\documentclass{article}

\usepackage[final,nonatbib]{neurips_2020}

\usepackage[dvipsnames]{xcolor}

\usepackage[numbers,sort&compress]{natbib}

\newcommand{\ie}{\emph{i.e.}} 
\newcommand{\eg}{\emph{e.g.}}

\newcommand\ci{\perp\!\!\!\perp}

\usepackage{amsmath,amsthm,amssymb}
\usepackage{color}
\usepackage{bm}
\usepackage{algorithm}
\usepackage{algorithmic}
\usepackage{hyperref}
\usepackage[inline]{enumitem}
\usepackage{mathtools}
\usepackage{appendix}
\usepackage{pgfplots}
\usepackage{mathrsfs}  
\usepackage{multicol}
\usepackage{multirow}
\usepackage{booktabs}
\pgfplotsset{compat=newest}
\usepackage{thmtools}
\usepackage{thm-restate}
\usepackage{cleveref}
\usepackage{caption}
\usepackage{subcaption}
\usepackage{wrapfig}
\usepackage{cite}

\newcommand{\inx}{\ensuremath{\mathcal{X}}}
\newcommand{\iny}{\ensuremath{\mathcal{Y}}} 
\newcommand{\inz}{\ensuremath{\mathcal{Z}}}  
\newcommand{\inw}{\ensuremath{\mathcal{W}}}

\newcommand{\pp}[1]{\ensuremath{\mathbb{#1}}}     % probability measure

\newcommand{\hbspf}{\ensuremath{\mathcal{F}}}
\newcommand{\hbspg}{\ensuremath{\mathcal{G}}}
\newcommand{\hbspu}{\ensuremath{\mathcal{U}}}

\newcommand{\rr}{\mathbb{R}} 		         % the real numbers
\newcommand{\ep}{\mathbb{E}}                     % the expectation

\newcommand{\kmat}{\mathbf{K}}                   % kernel matrix 
\newcommand{\lmat}{\mathbf{L}}                   % kernel matrix
\newcommand{\avec}{\bm{\alpha}}                  % the weight vector \alpha

\newcommand{\bvec}{\bm{\beta}}                   % the weight vector \beta

\newcommand{\covw}{\ensuremath{\mathcal{C}_{\mathit{W}}}}
\newcommand{\covxw}{\ensuremath{\mathcal{C}_{\mathit{XW}}}}
\newcommand{\covwx}{\ensuremath{\mathcal{C}_{\mathit{WX}}}}

\newcommand{\covr}{\ensuremath{\mathcal{R}}}
\newcommand{\ecovr}{\ensuremath{\widehat{\mathcal{R}}}}
\newcommand{\covrl}{\ensuremath{\mathcal{R}_{\lambda_1}}}
\newcommand{\ecovrl}{\ensuremath{\widehat{\mathcal{R}}_{\lambda_1}}}

\newcommand{\ecovw}{\ensuremath{\widehat{\mathcal{C}}_{\mathit{W}}}}
\newcommand{\ecovwx}{\ensuremath{\widehat{\mathcal{C}}_{\mathit{WX}}}}
\newcommand{\ecovxw}{\ensuremath{\widehat{\mathcal{C}}_{\mathit{XW}}}}

\newcommand{\covl}{\ensuremath{\mathcal{C}_{\lambda_1}}}
\newcommand{\ecovl}{\ensuremath{\widehat{\mathcal{C}}_{\lambda_1}}}
\newcommand{\bb}{\ensuremath{\mathbf{b}}}
\newcommand{\bh}{\ensuremath{\hat{\mathbf{b}}}}

\newcommand{\Id}{\ensuremath{\mathcal{I}}}

\newcommand{\y}{\ensuremath{\mathbf{y}}}

\newcommand{\M}{\ensuremath{\mathbf{M}}}

\hypersetup{
    colorlinks=true,
    citecolor=MidnightBlue
}

\usepackage{tikz}
\usepackage{xcolor}
\usepackage{etoolbox}
\newtoggle{quickdecim}

\usetikzlibrary{decorations}
\usetikzlibrary{decorations.pathreplacing}
\usetikzlibrary{calc}
\usetikzlibrary{arrows}

\newtheorem{theorem}{Theorem}
\newtheorem{definition}[theorem]{Definition}

\newtheorem{assumption}{Assumption}

\allowdisplaybreaks

\title{Dual Instrumental Variable Regression}

\author{%
  Krikamol Muandet \\
  Max Planck Institute for Intelligent Systems \\
  T\"ubingen, Germany \\
  \url{krikamol@tuebingen.mpg.de} \\
  \And
  Arash Mehrjou \\
  Max Planck Institute for Intelligent Systems \\
  ETH Z\"urich, Z\"urich, Switzerland \\
  \url{arash.mehrjou@inf.ethz.ch} \\
  \AND
  Si Kai Lee \\
  Booth School of Business \\ 
  University of Chicago, USA \\
  \url{sikai.lee@chicagobooth.edu} \\
  \And
  Anant Raj \\
  Max Planck Institute for Intelligent Systems \\
  T\"ubingen, Germany \\
  \url{anant.raj@tuebingen.mpg.de}
}

\begin{document}
\maketitle

\begin{abstract}
We present a novel algorithm for non-linear instrumental variable (IV) regression, DualIV, which simplifies traditional two-stage methods via a dual formulation. 
Inspired by problems in stochastic programming, we show that two-stage procedures for non-linear IV regression can be reformulated as a convex-concave saddle-point problem. 
Our formulation enables us to circumvent the first-stage regression which is a potential bottleneck in real-world applications. 
We develop a simple kernel-based algorithm with an analytic solution based on this formulation. 
Empirical results show that we are competitive to existing, more complicated algorithms for non-linear instrumental variable regression.
\end{abstract}

%%%%
\section{Introduction}
\label{sec:introduction}

Inferring causal relationships under the influence of unobserved confounders remains one of the most challenging problems in economics, health care, and social sciences \citep{Angrist08:Harmless,Imbens15:CIS}. 
% Recently, hidden confounders have also been recognized as one of the major obstacles in building more fair, accountable, and transparent machine learning models \citep{Madras19:FTC,Khademi19:FAD,Kilbertus19:Fairness}.
% In this work, we look into how one could uncover the causal relation between a treatment variable $X$ and an outcome variable $Y$ in the presence of an unobserved confounder $H$.
% In economics, studies on the return from schooling \citep{Card99:Labor} which attempt measure the causal effect of education on labor market earnings, is a typical example.
A typical example in economics is the study of returns from schooling \citep{Card99:Labor}, which attempts to measure the causal effect of education on labor market earnings.
For each individual, the treatment variable $X$ represents the level of education and the outcome $Y$ represents how much they earn. 
% Understanding the causal effect of education on earnings in this context could potentially lead to improved educational policies.
However, one's level of education and income is likely confounded by the socioeconomic status or other unobserved confounding factors $H$ \citep[Ch. 4]{Angrist08:Harmless}.  
% While several approaches have been proposed under the unconfoundedness assumption, causal inference without such an assumption is a widely unexplored territory. 
%Consequently, we cannot distinguish the true causal effect of $X$ on $Y$ from the effect that comes from $H$, unless the level of education is assigned randomly.
% This bias is known in the econometrics literature as the omitted variable bias (OVB) which, as the name suggests, results in a non-vanishing bias of our causal estimand.

%The most powerful strategy in dealing with hidden confounders is a randomized controlled trial (RCT). 
%Unfortunately, random assignment of treatment can be harmful or deems inappropriate. 
%A randomized experiment on the return to schooling is considered unethical by today standard.
%When the RCT is not feasible, one must resort to an observational study in which treatment assignment is not guaranteed to be free from hidden confounders.
%To deal with non-random treatment assignment, a set of control covariates, \eg, gender and race, or proxies of hidden confounders is often included as part of the causal analysis.

Since randomized control trials are often infeasible in most economic studies, 
economists have turned to \emph{instrumental variables} (IVs) or \emph{instruments} derived from naturally occurring random experiments to overcome unobserved confounding. 
%A more feasible approach is to use \emph{instrumental variables} (IVs) to correct for confounding in observational studies.
Informally, instrumental variables $Z$ are defined as variables that are associated with the treatment $X$, affect the outcome $Y$ only through $X$ and do not share common causes with $Y$.
For instance, the season-of-birth was used as an instrument in \citep{Angrist91:Schooling} to estimate the impact of compulsory schooling on earnings. 
Because of the compulsory school attendance laws, an individual's season-of-birth, which is likely to be random, affects how long they actually remain in school, but not their earnings.
Figure \ref{fig:dgp} illustrates this example.
Finding valid instruments for specific problems is an essential task in econometrics \citep{Angrist08:Harmless} and  epidemiology \citep{Burgess17:Mendelian}.

Although IV analysis is widely used, the statistical tools employed for estimating causal effect are fairly rudimentary.
Most applications of instrumental variables utilise a two-stage procedure \citep{Angrist08:Harmless,White82:IV,Singh19:KIV,Hartford17:DIV}. 
For instance, the two-stage least squares (2SLS) relies on the assumption that the relationship between $X$ and $Y$ is linear \citep{peters2017elements}.
It first estimates the conditional mean $\mathbb{E}[X|Z=z]$ via linear regression and then regresses $Y$ on the estimate of $\mathbb{E}[X|Z=z]$ to obtain an estimate of the causal effect.
Since the first-stage estimate is by construction independent from confounders, the resultant causal estimate is therefore free from hidden confounding.
In the non-linear setting, however, a poorly-fitted first-stage regression may result in inaccurate second-stage estimates \citep[Ch. 4.6]{Angrist08:Harmless}.

In this paper, we propose a novel procedure, DualIV, to directly estimate the structural causal function.
Unlike previous works which extend 2SLS by employing non-linear models in place of their linear counterparts \citep{Hartford17:DIV,Singh19:KIV}, we solve the dual problem which can be expressed as a convex-concave saddle-point problem. 
Based on this framework, we develop a consistent reproducing kernel Hilbert spaces-based (RKHS) algorithm.
Our formulation was inspired by the mathematical resemblance of non-linear IV to two-stage problems in stochastic programming \citep{Shapiro14:LSP,Dai17:Dual,Hsu19}.

The rest of the paper is organized in the following manner.
Section \ref{sec:preliminaries} introduces the IV regression problem, reviews related work and identifies current limitations.
We present our formulation in Section \ref{sec:dualiv-regression}, followed by the kernelized estimation method in Section \ref{sec:estimation}. 
Then, Section \ref{sec:experiments} reports empirical comparisons between DualIV and existing algorithms.
Finally, we discuss the limitations of our procedure and suggest future directions in Section \ref{sec:discussion}.
All proofs can be found in Appendix \ref{sec:proofs}.

%%%%
\section{Instrumental variable regression}
\label{sec:preliminaries}

Let $X$, $Y$, and $Z$ be treatment, outcome, and instrumental variable(s) taking values in $\inx$, $\iny$, and $\inz$, respectively.
In this work, we assume that $\iny \in \rr$, and $\inx$ and $\inz$ are Polish spaces.
We also assume that $Y$ is bounded, \ie, $|Y| < M < \infty$ almost surely. 
Moreover, we denote unobserved confounder(s) by $H$. 
The underlying data generating process (DGP) is described by the causal graph in Figure \ref{fig:dgp} equipped with the following structural causal model (SCM):
\begin{equation}\label{eq:sem}
    Y = f(X) + \varepsilon, \quad \mathbb{E}[\varepsilon]=0, 
\end{equation}
\noindent where $f$ is an unknown, potentially non-linear continuous function and $\varepsilon$ denotes the additive noise which depends on the hidden confounder(s) $H$. 
If $\mathbb{E}[\varepsilon \,|\, X]=0$, we can estimate $f$ consistently from observational data via the standard least-square regression.
This allows us to identify $\mathbb{E}[Y \,|\, \text{do}(X=x)]$ where $\text{do}(X=x)$ represents an intervention on $X$ where its value is set to $x$ \citep{Pearl2000}.

For non-expert readers, we elaborate that $\text{do}(X=x)$ here denotes a mathematical operator which simulates physical interventions by setting the value of $X$ to $x$, while keeping the rest of the model unchanged \citep[Sec. 3.2.1]{Pearl09:Causal}. 
That is, the conditional expectation $\mathbb{E}[Y\,|\,\text{do}(X=x)]$ is computed with respect to the interventional distribution $\pp{P}(Y\,|\,\text{do}(X=x))$.
We can estimate $\pp{P}(Y\,|\,\text{do}(X=x))$ if it is possible to directly manipulate $X$ and then observe the resulting outcome $Y$. 
In Figure \ref{fig:dgp}, for instance, one may assign different levels of education to people and then observe their subsequent levels of income in the labor market. 
Unfortunately, such experiment is not always possible and we only have access to an observational distribution $\pp{P}(Y\,|\,X=x)$, which can be different from $\pp{P}(Y\,|\,\text{do}(X=x))$.
The discrepancy between interventional and observational distributions may result from the unobserved socioeconomic status, as illustrated in Figure \ref{fig:dgp}.

%%%%
\begin{figure}[t!]
  \centering
  \begin{tikzpicture}[scale=0.45]
    \node[draw, circle, very thick] at (10,0) (uu) {$H$};
    \node[draw, circle, very thick, fill=blue!20] at (5.5,0) (y) {$Y$};
    \node[draw, circle, very thick, fill=blue!20] at (1,0) (x) {$X$};
    \node[draw, circle, very thick, fill=blue!20] at (-4,0) (z) {$Z$};
    \draw[->, very thick] (x) to (y);
    \draw[->, very thick, dashed] (uu) to [out=-150,in=-30] (x);
    \draw[->, very thick, dashed] (uu) to (y);
    \draw[->, very thick] (z) to (x);

    \node[] at (1,1.5) {\texttt{education}};
    \node[] at (-5,1.5) {\texttt{season of birth}};
    \node[] at (5.5,1.5) {\texttt{income}};
    \node[] at (13,1.5) {\texttt{socioeconomic status}};
  \end{tikzpicture} 
  \caption{A data generating process (DGP) with a hidden confounder $H$ and an instrument $Z$. A variation in $X$ comes from both $H$ and $Z$. Intuitively speaking, the external source of variation from $Z$ can help improve an estimation by removing the effect of $H$ on $X$.}
  \label{fig:dgp}
  \vspace{-10pt}
\end{figure}
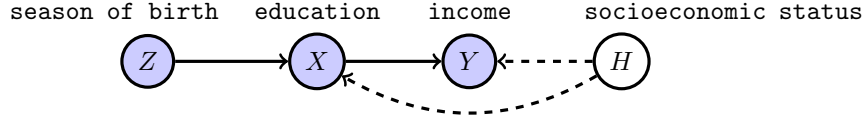

When hidden confounders exist between $X$ and $Y$, the error term $\varepsilon$ in \eqref{eq:sem} is generally correlated with $X$.
Hence, $\mathbb{E}[\varepsilon|X]\neq 0$ and it follows from \eqref{eq:sem} that
\begin{equation}
    \mathbb{E}[Y\,|\,X=x] = f(x) + \mathbb{E}[\varepsilon \,|\, X=x],
\end{equation}
\noindent which implies that $\mathbb{E}[Y \,|\, \text{do}(X=x)] \neq \mathbb{E}[Y\,|\, X=x]$. Thus, standard least-square regression no longer provides a valid estimate of $f$ for making a prediction about the outcome of an intervention on $X$ \citep{Newey03:NIV,Hartford17:DIV,Singh19:KIV}. 
To handle hidden confounders, we assume access to an instrumental variable(s) $Z$ which satisfies the following assumptions:
% (see, also, Figure \ref{fig:dgp}):
%%% assumption
%\begin{assumption}\label{asmp:main}
\begin{enumerate*}[label=(\roman*)]
    \item \label{asmp:relevance} \textbf{Relevance:} $Z$ has a causal influence on $X$.
    \item \label{asmp:exclusion} \textbf{Exclusion restriction:} $Z$ affects $Y$ only through $X$, \ie, ${Y \ci Z} | X,\varepsilon$.
    \item \label{asmp:unconfounded} \textbf{Unconfounded instrument(s):} $Z$ is independent of the error, \ie, ${\varepsilon \ci Z}$.
\end{enumerate*}
%\end{assumption}
%In practice, one may have access to the observed covariates $X'$ which can be incorporated using a modified causal model $Y = f(X,X')+\epsilon$.
%Controlling for such observed covariates may increase the plausibility of Assumption \ref{asmp:main}.   
%For clarity, we assume without loss of generality that potential confounders are not observed thoughout this paper.

The properties of $Z$ imply that $\mathbb{E}[\varepsilon \,|\,Z] = 0$. Taking the expectation of \eqref{eq:sem} w.r.t. $Y$ conditioned on $Z$ yields the following integral equation
\begin{equation}
    \label{eq:integral-eq}
    \mathbb{E}[Y | Z] = \int_{\inx} f(x) \,\mathrm{d}\pp{P}(x|Z),
\end{equation}
% From the functional analysis perspective, \eqref{eq:integral-eq} 
which is a Fredholm integral equation of the first kind.
Recent works in nonparametric IV regression have adopted this perspective \citep{horowitz2011applied,Newey03:NIV,Hartford17:DIV,Singh19:KIV} despite the fact that solving \eqref{eq:integral-eq} directly is an ill-posed problem as it involves inverting linear compact operators \citep{Kress89:LIE,Nashed74:Inverse,Newey03:NIV}.

To illustrate the role of an instrument, we consider two special cases. 
When $X$ is perfectly correlated with $Z$, the treatment is uncorrelated with the hidden confounder. 
In other words, we recover the strong ignorability assumption \citep{Rubin74:Causal,Rubin05:PO} required for causal inference.
When $Z$ is independent of $X$, the instrument is useless as it has no predictive power over treatment so the structural function $f$ is unidentifiable from the data.
Therefore, the most interesting cases lie between these two extremes, especially when $X$ and $Z$ are weakly correlated, see, \eg, \citep{Bound95:WeakInstrument,Staiger97:Weak}\citep[pp. 205--216]{Angrist08:Harmless}. 

% In particular, the problem of \emph{weak} instrument has received  considerable attention in the econometric literature \citep{Bound95:WeakInstrument,Staiger97:Weak}\citep[pp. 205--216]{Angrist08:Harmless}.\footnote{An instrumental variable $Z$ is said to be weak if it explains little of the variation in endogeneous variable $X$; See, \eg, \citep{Bound95:WeakInstrument,Staiger97:Weak} for more details.}

%%%%
\subsection{Previous work}
\label{sec:related-work}

Early applications of instrumental variables often assume linear relationships between $Z$ and $X$ as well as $X$ and $Y$ \citep{Angrist96:IV,Angrist08:Harmless}.
When there is a single endogeneous variable and instrument, the structural parameter can be estimated consistently by the instrumental variable (IV) estimator \citep{Angrist96:IV}.
Interestingly, we can obtain this estimate using a two-stage procedure: regress $X$ on $Z$ using ordinary least square (OLS) to calculate the predicted value of $X$ and used that as an explanatory variable in the structural equation to estimate the structural parameter using OLS. 
When there are multiple instruments, the two-stage least squares (2SLS) estimator is obtained by using all the instruments simultaneously in the first-stage regression. 
\citet[Theorem 5.3]{Wooldridge01:Econometric} asserts that the 2SLS estimator is the most efficient IV estimator; see, \eg, \citep{Wooldridge01:Econometric,Angrist08:Harmless} for a detailed exposition.

Recently, several extensions of 2SLS have been proposed to overcome the linearity constraint.
The first line of work replaces linear regression by a linear projection onto a set of known basis functions \citep{Newey03:NIV,Blundell07:Semi-NIV,horowitz2011applied,Chen12:NIV}.
\citet{Chen18:NIV-Uniform} provides a uniform convergence rate of this approach.
However, there exists no principled way of choosing the appropriate set of basis functions.
The second line of work replaces the first-stage regression by a conditional density estimate of $\mathbb{P}(X|Z)$ \citep{Hall05:NIV,Darolles11:NIV}.
Despite being more flexible, such approaches are known to suffer from the curse of dimensionality \citep[Ch. 1]{Tsybakov08:INE}.
Other extensions of 2SLS are DeepIV \citep{Hartford17:DIV} and KernelIV \citep{Singh19:KIV} algorithms.
In \citep{Hartford17:DIV}, \eqref{eq:integral-eq} is solved by first estimating $\mathbb{P}(X|Z)$ with a mixture of deep generative models on which $f$ is learned using another deep neural network.
Instead of neural networks, \citet{Singh19:KIV} proposes to model the first-stage regression using the conditional mean embedding of $\mathbb{P}(X|Z)$ \citep{Song10:KCOND,Song2013,Muandet17:KME} which is then used in the second-stage kernel ridge regression.

%%%
\vspace{-5pt}
\paragraph{The curse of two-stage methods.}
Two-stage procedures have two fundamental issues. 
First, such procedures violate Vapnik's principle \citep{Vapnik98:SLT}: ``\emph{[...] when solving a problem of interest, do not solve a more general problem as an intermediate step [...]}''.
Specifically, estimating the conditional density \citep{Hartford17:DIV} or the conditional mean embedding  \citep{Singh19:KIV} via regression in the first stage can be harder than estimating the parameter of interest in the second stage. 
The first stage is even referred as the ``forbidden regression'' in econometrics \citep[Ch. 4.6]{Angrist08:Harmless}.
On top of that, we usually only observe a single sample from each $\mathbb{P}(X|Z=z)$, which further increases the difficulty of the task.
Second, although two-stage procedures are asymptotically consistent, the first-stage estimate creates a finite-sample bias in the second-stage estimate \citep[Sec. 4.6.4]{Angrist08:Harmless}. 
This bias can be alleviated through sample splitting \citep{Angrist93:Split} which is also used in \citep{Hartford17:DIV,Singh19:KIV}.
Thus, two-stage procedures are less sample efficient and could yield biased estimates when run on the smaller datasets common in economics and social sciences.

The generalized method of moments (GMM) framework provides another set of popular approaches for estimating $f$ \citep{Hansen82:GMM,Hall05:GMM}.
Unlike two-stage procedures, GMM-based algorithms find a function $f$ that satisfies the orthogonality condition $\mathbb{E}[\varepsilon | Z]=0$ directly. 
Specifically, if $g_1,g_2,\ldots,g_m$ are arbitrary real-valued functions, the orthogonality condition implies that $\mathbb{E}[(Y - f(X))g_j(Z)] = 0$ for $j=1,\ldots,m$.
The GMM estimate of $f$ can then be obtained by minimizing the quadratic form $\frac{1}{2}\sum_{j=1}^m\psi(f,g_j)^2$ where $\psi(f,g) := \mathbb{E}[(Y - f(X))g(Z)]$.
This estimator can be interpreted as a generalization of the 2SLS estimator in the linear setting \citep{White82:IV}.
Recently, extensions of GMM-based methods where both $f$ and $g$ are parameterized by deep neural networks have successfully been used to solve non-linear IV regression \citep{Lewis18:AGMM,Bennett19:DeepGMM}. 
In contrast, \citet{Muandet20:KCM} considers the set of RKHS functions which allow for an analytic formulation of the orthogonality condition.

%%%%%%%%%%
\section{Dual IV}
\label{sec:dualiv-regression}

In this section, we reformulate the integral equation \eqref{eq:integral-eq} as an empirical risk minimization problem and present DualIV algorithm.

\subsection{Empirical risk minimization}

Let $\ell: \mathbb{R}\times\mathbb{R}\to\mathbb{R}_+$ be a proper, convex, and lower semi-continuous loss function for any value in its first argument.\footnote{The function $f$ is \emph{proper} if $\text{dom}\, f\neq\emptyset$ and $f(x)>-\infty, \,\forall x\in\inx$. It is  \emph{lower (upper) semi-continuous} at $x_0\in\inx$ if for $\varepsilon > 0$ there exists a neighborhood $N(x_0)$ of $x_0$ such that $\varepsilon < (>) f(x)-f(x_0)$ for all $x\in N(x_0)$ \citep{Rockafellar70:Convex}.}
Let $\mathcal{F}$ be an arbitrary class of continuous functions which we assume contains $f$ that fulfills the integral equation \eqref{eq:integral-eq}.
Then, we can formulate \eqref{eq:integral-eq} as
\begin{equation}\label{eq:erm}
    \min_{f\in \mathcal{F}} \, R(f) 
    := \ep_{\mathit{YZ}}\left[\ell(Y,\mathbb{E}_{X|Z}[f(X)]) \right],
\end{equation}
where $R(f)$ denotes the expected risk of $f$. 
To understand how \eqref{eq:integral-eq} and \eqref{eq:erm} are related, let us consider the squared loss $\ell(y,y') = (y-y')^2$ and define $h(z) := \mathbb{E}_{X|z}[f(X)]$. 
Then, the solution to \eqref{eq:erm} is the minimum mean square error (MMSE) estimator $h^{*}(z) := \mathbb{E}[Y|z]$, which is exactly the LHS of \eqref{eq:integral-eq}. 
If there exists no $f\in\mathcal{F}$ for which $h^{*}(z)=\mathbb{E}[Y|z]$, we use $h^{*}(z)$ as the best MMSE approximation.

%Note that our formulation \eqref{eq:erm} differs \ar{Actually I think this is "Learning with Invariance" which is a special case of  equation (1) in \citet{Dai17:Dual}. It can also be seen as "Virtual Samples technique"\citet{niyogi1998incorporating} for incorporating prior knowledge} \krik{From the causal perspective, I do not see how we can interpret the instrumental analysis as learning with invariance. What is the virtual sample here and how can we interpret it from a causal perspective? Could you elaborate?}\ar{I saw the connection only from mathematical point of view. I am not sure about the interpretation. If you look at eq (2) of Dai et al 2017, the difference with the general formulation  eq(1) is that the second argument of $\mathbb { E } _ { x , y } \left[ \ell \left( y , \mathbb { E } _ { z | x \sim \mu ( g ( x ) ) } \left[ \langle f , \psi ( z ) \rangle _ { \tilde { \mathcal { H } } } \right] \right) \right]$ does explicitly depend on $x$ which is the condition of the conditional distribution $z | x \sim \mu ( g ( x ) )$. Similarly in our case, the function $g(t, x)$ does not depend on the instrumental variable $v$.}

The key challenge is if $f$ is noncontinuous in $h(z)$, it is not guaranteed to be consistently estimated even if $h(z)$ is estimated correctly \citep{Newey03:NIV}.
% the consistency of estimating $h(z)$ does not imply the consistency of estimating $g$ \citep{Newey03:NIV}.
We defer further discussion to Section \ref{sec:identifiability}.
In addition, it remains cumbersome to solve \eqref{eq:erm} directly because of the inner expectation.
To circumvent this, we look to similar two-stage problems in stochastic programming \citep{Shapiro14:LSP,Dai17:Dual}.
% , see, also, \citep{Dai17:Dual}.
For example, in \citep{Dai17:Dual}, the problem of learning from conditional distributions was formulated in a similar fashion to \eqref{eq:erm}.
Moreover, \citep{Hsu19} proposes the deconditional mean embedding (DME) which solves the integral equation \eqref{eq:integral-eq} by performing a closed-form ``inversion'' of the conditional mean embedding of $\pp{P}(X|Z)$ (see \citep{Muandet17:KME,Song13:CME} for a review).
% The main advantage of DME is that it has a closed-form solution.
In contrast, we solve the equivalent dual formulation of \eqref{eq:erm} instead of \eqref{eq:integral-eq}.

%%%
\subsection{Dual formulation}

To derive the dual of \eqref{eq:erm}, we employ two existing results, \emph{interchangeability} and \emph{Fenchel duality}, which we review; see, \eg, \citep[Lemma 1]{Dai17:Dual}, \citep[Ch. 14]{Rock98:VA}, and \citep[Ch. 7]{Shapiro14:LSP} for more details. 

%%%%%%
\begin{theorem}[Interchangeability]\label{thm:interchange}
Let $\omega$ be a random variable on $\Omega$ and, for any $\omega\in\Omega$, the function $f(\cdot,\omega):\mathbb{R}\to (-\infty,\infty)$ is proper and upper semi-continuous concave function. Then,
\begin{equation}
    \mathbb{E}_{\omega}\left[\max_{u\in\mathbb{R}} f(u,\omega)\right] = \max_{u(\cdot)\in\mathcal{U}(\Omega)}\mathbb{E}_{\omega}[f(u(\omega),\omega)],
\end{equation}
\noindent where $\mathcal{U}(\Omega) := \{u(\cdot): \Omega\to\mathbb{R}\}$ is the entire space of functions defined on the support $\Omega$. 
\end{theorem}

%%%%%%
\begin{definition}[Fenchel duality]\label{def:fenchel}
Let $\ell: \mathbb{R}\times\mathbb{R}\to\mathbb{R}_+$ be a proper, convex, and lower semi-continuous loss function for any value in its first argument and $\ell^{\star}_y := \ell^{\star}(y,\cdot)$ a convex conjugate of $\ell_y := \ell(y,\cdot)$ which is also proper, convex, and lower semi-continuous w.r.t. the second argument. 
Then, $\ell_y(v) = \max_u\{uv - \ell_y^{\star}(u)\}$. 
The maximum is achieved at $v\in\partial\ell^{\star}(u)$, or equivalently $u\in\partial\ell(v)$.
\end{definition}

Applying the interchangeability and Fenchel duality to \eqref{eq:erm} yields the expected loss
\begin{eqnarray*}
 R(f) &=& \mathbb{E}_{\mathit{YZ}}[\max_{u\in\mathbb{R}}\{\mathbb{E}_{X|Z}[f(X)] u - \ell^{\star}_Y(u)\}] \\
      &=& \max_{u\in\mathcal{U}}\; \mathbb{E}_{\mathit{YZ}}[\mathbb{E}_{X|Z}[f(X)] u(Y,Z)
        - \ell^{\star}_Y(u(Y,Z))] \\
      &=& \max_{u\in\mathcal{U}}\; \mathbb{E}_{\mathit{XYZ}}[f(X) u(Y,Z)] - \mathbb{E}_{\mathit{YZ}}[\ell^{\star}_Y(u(Y,Z))]
\end{eqnarray*}
where $\mathcal{U}$ is the space of continuous functions over $\mathcal{Y}\times\mathcal{Z}$. 
Hence, \eqref{eq:erm} can be reformulated as 
\begin{equation}
    \label{eq:saddle-point-form}
    \min_{f\in\mathcal{F}}\max_{u\in\mathcal{U}} \,
     \mathbb{E}_{\mathit{XYZ}}\left[f(X)u(Y,Z)\right] - \mathbb{E}_{\mathit{YZ}}\left[\ell^{\star}_Y(u(Y,Z))\right].
\end{equation}
Following \citep{Dai17:Dual}, we will refer to $u\in\mathcal{U}$ as the \emph{dual function}.
Note that this function depends on only the outcome $Y$ and the instrument $Z$, but not the treatment $X$.

The advantages of our formulation \eqref{eq:saddle-point-form} over \eqref{eq:integral-eq} and \eqref{eq:erm} are twofold. 
First, there is no need to estimate $\mathbb{E}_{X|Z}[f(X)]$ or $\pp{P}(X|Z)$ explicitly.
Second, the target function $f$ appears linearly in \eqref{eq:saddle-point-form} which makes it convex in $f$.
Since $\ell^{\star}_y$ is also convex, \eqref{eq:saddle-point-form} is concave in the dual function $u$.
Hence, \eqref{eq:saddle-point-form} is essentially a convex-concave saddle-point problem for which efficient solvers exist \citep{Dai17:Dual}.

For the squared loss $\ell(y,y') = (y-y')^2$, we have $\ell^{\star}_y(w) = wy + \frac{1}{2}w^2$ (see Appendix \ref{sec:conjugate-loss} for the derivation) and the saddle-point problem \eqref{eq:saddle-point-form} reduces to
\begin{align}
    \label{eq:1sls-obj}
    \min_{f\in\mathcal{F}}\max_{u\in\mathcal{U}} \; \Psi(f,u) &:=
    \mathbb{E}_{\mathit{XYZ}}\left[(f(X)-Y)u(Y,Z)\right] - \frac{1}{2}\mathbb{E}_{\mathit{YZ}}\left[u(Y,Z)^2\right].
\end{align}
To solve \eqref{eq:1sls-obj}, one can adopt an SGD-based algorithm developed by \citet{Dai17:Dual}. 
Alternatively, we propose in Section \ref{sec:estimation} a simple algorithm that can solve \eqref{eq:1sls-obj} in closed form.
%Given an i.i.d. sample $\{(x_i,y_i,z_i)\}_{i=1}^n$ from the joint distribution $\mathbb{P}(X,Y,Z)$, the empirical version of \eqref{eq:saddle-point-form} is given by
%\begin{equation*}
%    \min_{g\in\mathcal{G}}\max_{u\in\mathcal{U}}\,
%    \frac{1}{n}\sum_{i=1}^n(g(x_i) u(y_i,z_i)) - \frac{1}{n}\sum_{i=1}^n\ell^{\star}_y(u(y_i,z_i)).
%\end{equation*}
% We will focus on this formulation throughout.
%Figure \ref{fig:dual-loss} illustrates the loss surface associated with \eqref{eq:1sls-obj}.

%Theorem \ref{thm:gmm-is-dualiv} links our formulation---which is the dual problem---to the GMM estimate as well as its recent nonparametric extensions \citep{Lewis18:AGMM,Bennett19:DeepGMM}.
% of the two-stage procedure
%For example, if we let $u(y,z) = u(z)$, \ie, the dual function depends only on the instrumental variable $Z$, then \eqref{eq:1sls-obj} takes the form of vanilla GMM.
%In particular, the resulting formulation resembles that in \citet[Eq. 3]{Lewis18:AGMM}.
%On the other hand, this result also highlights their fundamental differences as $f$ depends on both $Y$ and $Z$.
%We leave a thorough analysis for future work.

%%%%
\subsection{Interpreting the dual function}
\label{sec:interpretation}

The dual function $u(y,z)$ plays an important role in our framework.
To understand its role, we consider the minimization and maximization problems in \eqref{eq:1sls-obj} separately. 
For any $f\in\mathcal{F}$, the maximization problem is $\max_{u\in\mathcal{U}} 
\mathbb{E}_{\mathit{XYZ}}[(f(X) - Y) u(Y,Z)] - \frac{1}{2}\|u(Y, Z)\|_{L^2(\pp{P}_{\mathit{YZ}})}^2$ where the first term can be viewed, loosely speaking, as a loss function and the second as a regularizer. 
% where the first term can be viewed as a (negative) loss function and the second term as a regularizer. % Sky: first term can positive or negative depending on the sign of noise added
Intuitively, we are seeking $u^{*}\in\mathcal{U}$ that is least orthogonal to the residual.
Given $u^{*}$, the outer minimization problem $\min_{f\in\mathcal{F}} \;
\mathbb{E}_{\mathit{XYZ}}\left[(f(X)-Y) u^{*}(Y,Z)\right]$ finds the function $f$ that yields the most orthogonal residual to $u^{*}$.
Our procedure clearly differs from previous two-stage methods as the minimization and maximization stages are interdependent.

Examining the formulation in the context of instrumental variable regression, the residual contains the variation that cannot be explained by the current estimate of $f$ due to hidden confounding.
We select $u$ that maximally reweights the residuals according to how inconsistent they are w.r.t. the unconfounded joint distribution of $Y$ and $Z$.
% $u(Y, Z)$ is unconfounded with $H$ as $X$ is a unobserved collider.
%A crucial distinction between the proposed method and two-stage procedures is the inclusion of $Y$ in $u$: without $Y$, we are unable to determine the quality our predictions conditioned on $Z$ which are not influenced by hidden confounders by definition. 
%We believe that this could be one of the main reasons for why our method performs better than two-stage procedures.
Given $u$, we then select $f$ that minimizes the inconsistencies between the residuals and $u$.
Hence, at the equilibrium, we are left with residuals uncorrelated with $Y$ and $Z$ which can be attributed to noise due to unobserved confounding.

Lastly, we draw a connection between \eqref{eq:1sls-obj} and GMM.
Let $g_1,g_2,\ldots,g_m$ be real-valued functions on $\iny\times\inz$ and $\psi(f,g) := \ep[(Y-f(X))g(Y,Z)]$. 
When $\mathcal{U} = \text{span}\{g_1,\ldots,g_m\}$, it is not difficult to show that $\max_{u\in\mathcal{U}} \Psi(f,u) = \frac{1}{2}\bm{\psi}^\top\Lambda^{-1}\bm{\psi}$
where $\Lambda := \mathbb{E}_{\mathit{YZ}}[\mathbf{g}\otimes\mathbf{g}]$ with $\mathbf{g} := (g_1(Y,Z),\ldots,g_m(Y,Z))^\top$; see Appendix \ref{sec:gmm-is-dualiv}.
That is, minimizing the above over $f$ yields a formulation that strongly resembles the GMM objective, with the dual function $u(Y,Z)$ playing a role similar to that of an instrument. However, we must clarify that $u$ cannot act as an instrument since it depends on $Y$ and thereby violates the exclusion restriction assumption.
% This ambiguity has been resolved in \citet[Appendix F]{Liao20:NeuralSEM} who resorting to an alternative formulation similar to \eqref{eq:erm} and \eqref{eq:1sls-obj}.
\citet[Appendix F]{Liao20:NeuralSEM}, using an alternative formulation similar to \eqref{eq:erm} and \eqref{eq:1sls-obj},  showed that one can obtain a dual function $u(Z)$ that can act as an instrument.  
Furthermore, we also note that AGMM \citep{Lewis18:AGMM} and DeepGMM \citep{Bennett19:DeepGMM} rely on minimax optimization, similar to \eqref{eq:1sls-obj}, but were formulated based on the GMM framework.

\subsection{Theoretical analysis}
\label{sec:identifiability}

This section provides the conditions for which the true structural function $f^*$ can be identified by the optimum of the saddle-point problem \eqref{eq:1sls-obj}. 
We lay out the assumptions needed for the optimal dual function $u^{*}$ to be unique and continuous, show that the saddle-point formulation \eqref{eq:1sls-obj} is equivalent to the problem \eqref{eq:erm} under the squared loss and prove that the solution of \eqref{eq:1sls-obj} given $u^{*}$ is indeed $f^*$.

%%%
\begin{assumption}\label{asmp:identifiability}
\begin{enumerate*}[label=(\roman*)]
    %\item $g(x)$ is continuous in $x$.
    \item \label{asmp:cont} $\pp{P}(X|Z)$ is continuous in $Z$ for any values of $X$.
    \item \label{asmp:realize} The function class $\mathcal{F}$ is correctly specified, \ie, $f^*\in\mathcal{F}$.
\end{enumerate*}
\end{assumption}

Following \citep{Dai17:Dual}, we define the optimal dual function for any pair $(y,z)\in\iny\times\inz$ as
$u^*(y,z) \in \arg\max_{u\in\rr} \{ \ep_{X|z}[f(X)-y]u - (1/2)u^2 \}$.
Since this is an unconstrained quadratic program, $u^*(y,z)$ takes the form $\ep_{X|z}[f(X)] - y$. 
Given Assumption \ref{asmp:identifiability} and the loss function $\ell$ is convex and continuously differentiable, it follows from \citep[Proposition 1]{Dai17:Dual} that $u^*$ is unique and continuous.

Next, we shows that if $(f^*,u^*)$ is the saddle-point of \eqref{eq:1sls-obj}, $f^*$ minimizes the original objective \eqref{eq:erm}.
The result follows from plugging $u^* = \ep_{X|z}[f(X)]-y$ into the dual loss $\Psi(f,u)$ in \eqref{eq:1sls-obj}; see Appendix \ref{sec:proof-saddle-sol} for the detailed proof.

%%%
\begin{restatable}{proposition}{saddlepoint}
\label{prop:saddle-solution}
    Let $\ell(y,y') = \frac{1}{2}(y-y')^2$.
    Then, for any fixed $f$, we have $R(f) = \max_{u}\Psi(f,u)$.
\end{restatable}

By Proposition \ref{prop:saddle-solution} and the convexity of the loss $\ell(y,y')$, we obtain the following result.
\begin{theorem}\label{thm:identifiability}
    Let $\ell(y,y') = \frac{1}{2}(y-y')^2$ and assume that Assumption \ref{asmp:identifiability} holds. Then, $(f^*,u^*)$ is the saddle-point of a minimax problem $\min_{f\in\mathcal{F}}\max_{u\in\mathcal{U}}\,\Psi(f,u)$.
\end{theorem}

By virtue of Theorem \ref{thm:identifiability}, we can identify the true function $f^*$ under relatively weak assumptions.
In contrast, previous work usually require stronger assumptions such as the completeness condition \citep{Newey03:NIV,Singh19:KIV} which specifies that the first-stage conditional expectation $\mathbb{E}_{X|z}[f(X)]$ is injective, or $h(z) = \mathbb{E}_{X|z}[f(X)]$ is a smooth function of $z$ \citep{Singh19:KIV,Chen18:NIV-Uniform,Chen12:NIV}.
Since we do not perform first-stage regression, we only require $\pp{P}(X|Z)$ is continuous in $Z$ for any value of $X$.
The assumption that \eqref{eq:erm} is correctly specified, \ie, $f^{*}\in\mathcal{F}$, is standard in the literature \citep{Newey03:NIV,horowitz2011applied,Singh19:KIV}.

\begin{figure}[t!]
    \centering
    \includegraphics[height=0.3\columnwidth]{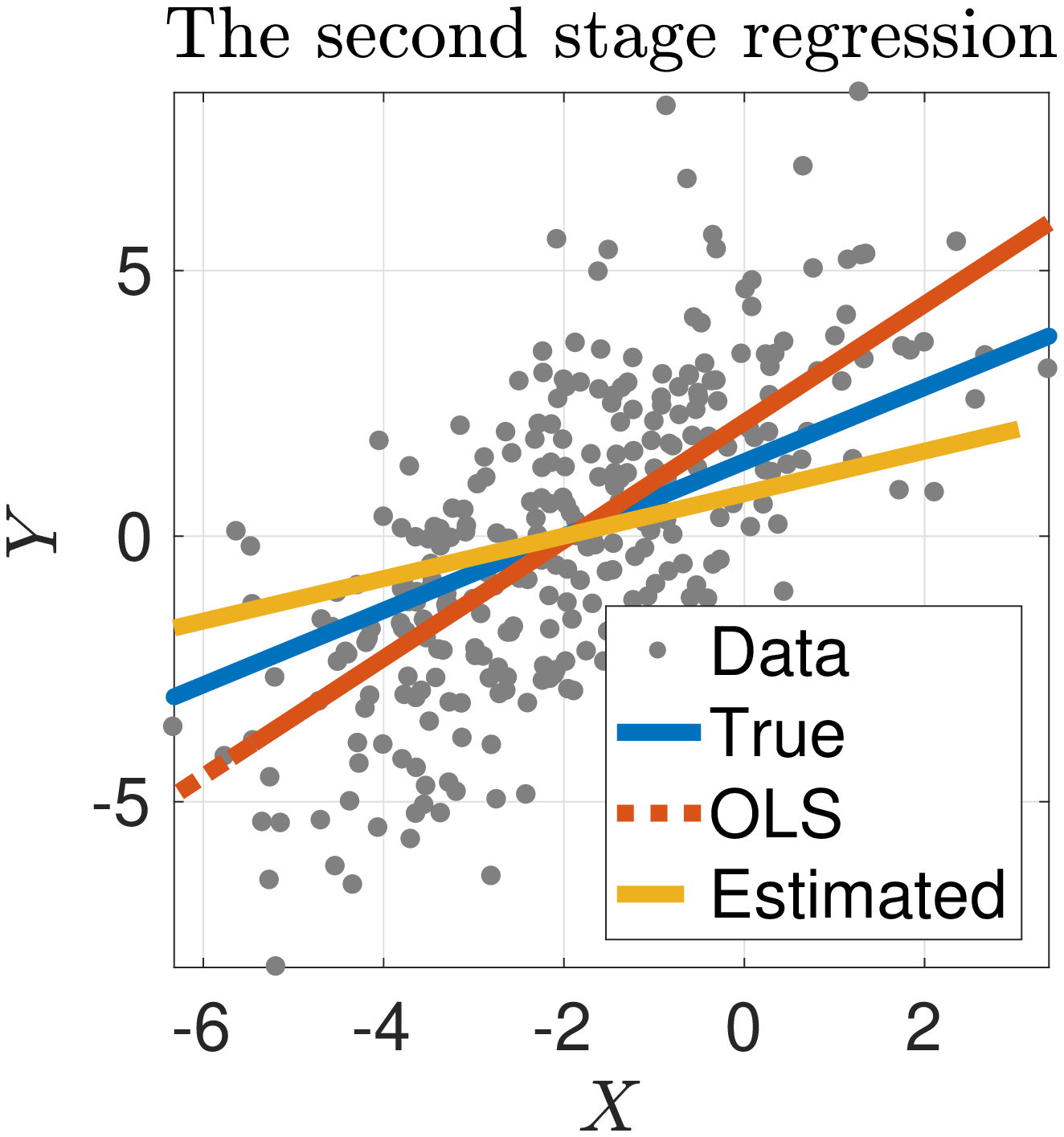}\hfill
    \includegraphics[height=0.3\columnwidth]{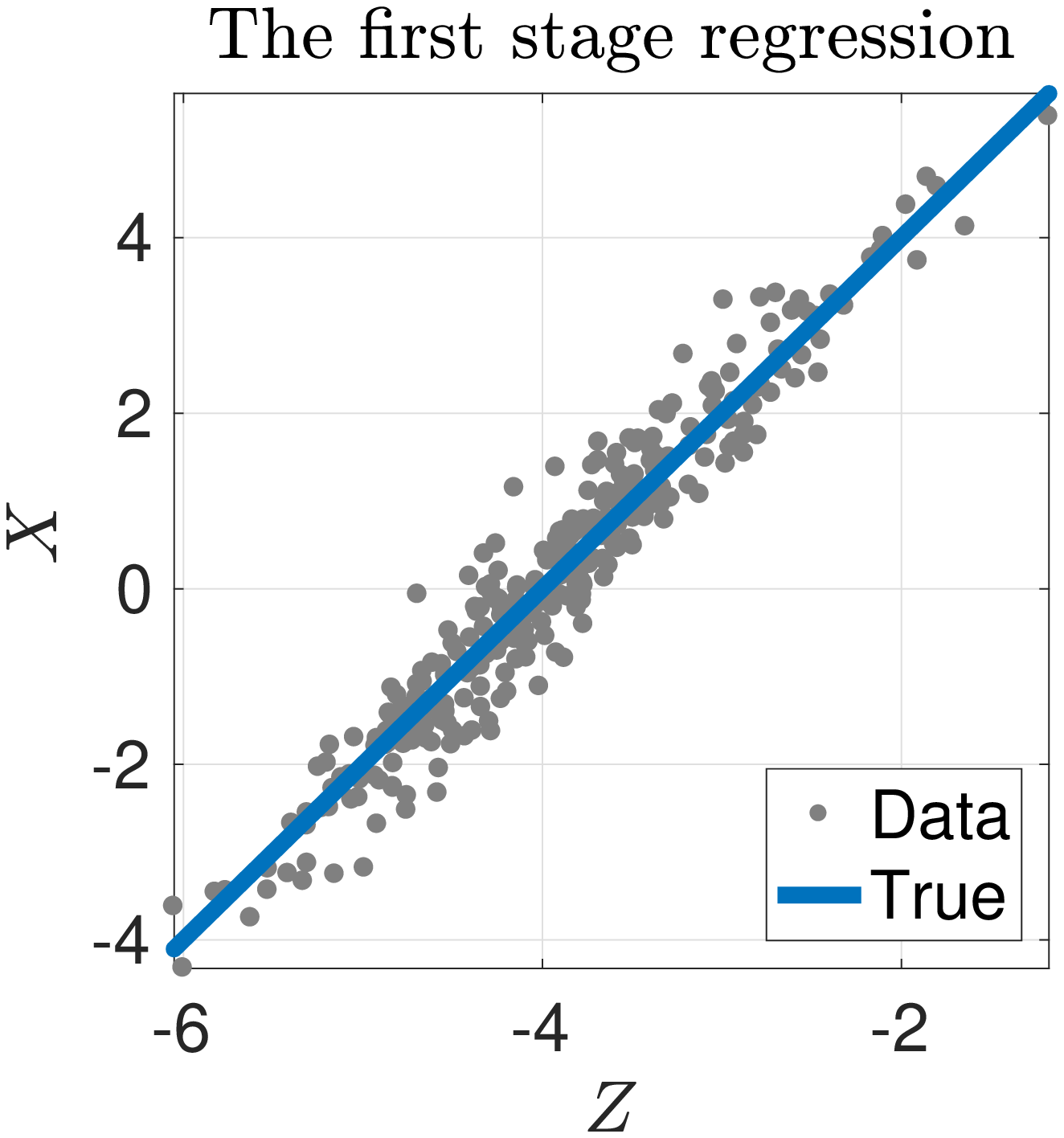}\hfill
    \includegraphics[height=0.3\columnwidth]{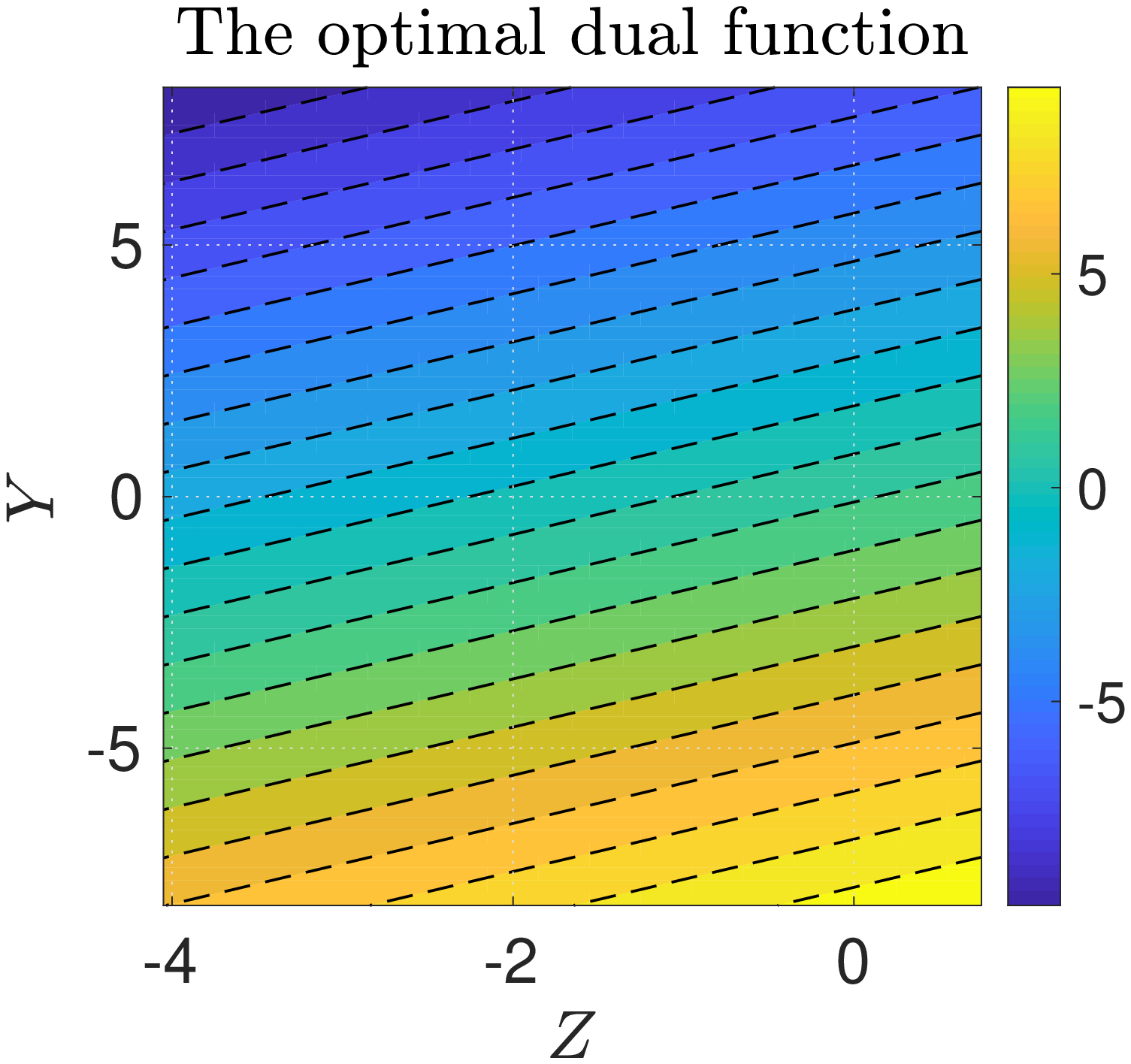}
    \caption{The dual function $u$ w.r.t. the current estimate $f$ in the linear setting \eqref{eq:linear-model}. For each $y$ and $z$, $u$ can directly measures the discrepancy between $y$ and $\ep_{X|z}[f(X)]$.}
    \label{fig:visualize-dual}
\end{figure}

As we can see, the optimal dual function $u^{*}(y,z) = \ep_{X|z}[f(X)]-y$ acts as a residual function measuring the discrepancy between $y$ and $\ep_{X|z}[f(X)]$ \citep{Dai17:Dual}.
Remarkably, this makes it possible to approximate $R(f)$ in \eqref{eq:erm} without computing the expectation $\ep_{X|z}[f(X)]$ explicitly.
We later exploit this property when performing hyperparameter selection.
Moreover, since $\ep_{X|z}[f(X)]$ allows $X$ and $Z$ to have a non-linear relationship, $u$ can be non-linear even when the true structural function $f^*$ is linear. 
This flexibility enables $u$ to accommodate a larger class of functions that maps $Z$ to $X$.
Figure \ref{fig:visualize-dual} illustrates this given the following generative process:
\begin{equation}\label{eq:linear-model}
    Y = X\beta + e + \epsilon, \quad X = (1-\rho)Z_1 + \rho e + \eta
\end{equation}
where $e \sim \mathcal{N}(0,2), Z_1\sim\mathcal{N}(0,2), \epsilon \sim \mathcal{N}(0,0.1)$, and $\eta\sim\mathcal{N}(0,0.1)$. 
The parameter $\rho$ controls the strength of the instrument w.r.t. hidden confounder $e$. Here, we set $n=300$, $\beta=0.7$, $\hat{\beta}=0.4$, and $\rho=0.2$ where $\hat{\beta}$ is an OLS estimate of $\beta$.
Under this model, we have $u^*(y,z) \approx \hat{\beta}(1-\rho)z - y$.

%%%%
\section{Kernelized DualIV}
\label{sec:estimation}

    To demonstrate the effectiveness of our framework, we develop a simple kernel-based algorithm using the new formulation \eqref{eq:1sls-obj}.
    To simplify notation, we denote by $W := (Y,Z)$ a random variable taking value in $\mathcal{W} := \mathcal{Y}\times\mathcal{Z}$.
    We pick $\hbspf$ and $\hbspu$ to be reproducing kernel Hilbert spaces (RKHSs) associated with positive definite kernels $k:\inx\times\inx\to\rr$ and $l:\inw \times \inw \to\rr$, respectively.
    Let $\phi: x\mapsto k(x,\cdot)$ and $\varphi: w\mapsto l(w,\cdot)$ be the canonical feature maps \citep{Scholkopf01:LKS}.
    We assume both $\hbspf$ and $\hbspu$ are \emph{universal} and hence are dense in the space of bounded continuous functions \citep[Ch. 4]{SteChr2008}.
    %Hence, both $\phi(x)$ and $\varphi(w)$ can be infinite dimensional feature maps, \eg, when $k$ and $l$ are Gaussian kernels.

    Then, for any $f\in\hbspf$ and $u\in\hbspu$, we can rewrite the objective in \eqref{eq:1sls-obj} as
    \begin{align}\label{eq:rkhs-obj}
        \Psi(f,u) &= \ep_{\mathit{XW}}[f(X)u(W)] - \ep_{\mathit{YZ}}[Yu(Y,Z)] - \frac{1}{2}\ep_{W}[u(W)^2] \nonumber \\
        %&= \ep_{\mathit{XW}}[\langle g,\phi(X)\rangle_{\hbspg}\langle u,\varphi(W)\rangle_{\hbspu}] - \ep_{\mathit{YZ}}[Y\langle u,\varphi(Y,Z) \rangle_{\hbspu}]
        %- \frac{1}{2}\ep_{W}[\langle u,\varphi(W) \rangle_{\hbspu}^2] \nonumber \\
        %&= \ep_{\mathit{XW}}[\langle g\otimes u,\phi(X)\otimes\varphi(W)\rangle_{\text{HS}}] \nonumber - \langle u,\ep_{\mathit{YZ}}[Y\varphi(Y,Z)] \rangle_{\hbspu} \nonumber \\
        %& \quad - \frac{1}{2}\ep_{W}[\langle u\otimes u,\varphi(W)\otimes \varphi(W) \rangle_{\text{HS}}] \nonumber \\
        %& = \langle \covxw,g\otimes u\rangle_{\text{HS}} - 
        %\langle u,\mathbf{b}\rangle_{\hbspu} 
        %- \frac{1}{2}\langle \covw,u\otimes u\rangle_{\text{HS}} \nonumber \\
        %& = \langle g, \covxw u\rangle_{\hbspg} - \langle u,\mathbf{b}\rangle_{\hbspu} - \frac{1}{2}\langle u,\covw u \rangle_{\hbspu} \nonumber \\
         &= \langle \covwx f - \mathbf{b}, u\rangle_{\hbspu} - \frac{1}{2}\langle u,\covw u \rangle_{\hbspu} ,
    \end{align}
    %\begin{align}\label{eq:rkhs-obj}
    %    \Psi(g,u) 
    %    &= \langle g, \covxw u\rangle_{\hbspg} - \langle u,\mathbf{b}\rangle_{\hbspu} - \frac{1}{2}\langle u,\covw u \rangle_{\hbspu} \nonumber \\
    %    &= \langle \covwx g - \mathbf{b}, u\rangle_{\hbspu} - \frac{1}{2}\langle u,\covw u \rangle_{\hbspu}
    %\end{align}
    \noindent where $\mathbf{b} := \ep_{\mathit{YZ}}[Y\varphi(Y,Z)]\in\hbspu$, $\covw := \ep_{W}[\varphi(W)\otimes\varphi(W)] \in\hbspu\otimes\hbspu$ is a covariance operator, and $\covwx := \ep_{\mathit{WX}}[\varphi(W)\otimes\phi(X)] \in \hbspu\otimes\hbspf$ is a cross-covariance operator \citep{Baker1973,Fukumizu04:DRS} (see Appendix \ref{sec:apx-kernelized-iv}). 
    Since \eqref{eq:rkhs-obj} is quadratic in $u$, we have $\covw u^* = \covwx f - \mathbf{b}$. 
    Substituting $u^*$ back into \eqref{eq:rkhs-obj} yields 
    \begin{align}\label{eq:close-form}
        f^* = \arg\min_{f\in\hbspf} \, \frac{1}{2}\left\langle \covwx f - \mathbf{b}, \covw^{-1}(\covwx f - \mathbf{b})\right\rangle_{\hbspu}
        = (\covxw\covw^{-1}\covwx)^{-1}\covxw\covw^{-1}\mathbf{b}.
    \end{align}
    We can view \eqref{eq:close-form} as a generalized least squares solution in RKHS.
    Since $\covw^{-1}$ and $(\covxw\covw^{-1}\covwx)^{-1}$ may not exist in general, we replace them with regularized versions $(\covw + \lambda_1\mathcal{I})^{-1}$ and $(\covxw\covw^{-1}\covwx + \lambda_2\mathcal{I})^{-1}$ where $\mathcal{I}$ is the identity operator and $\lambda_1,\lambda_2 > 0$ are regularization parameters.

    Given an i.i.d. sample $(x_i,y_i,z_i)_{i=1}^n$ from $\mathbb{P}(X,Y,Z)$, we define $\Phi := [\phi(x_1),\ldots,\phi(x_n)]$, $\Upsilon := [\varphi(y_1,z_1),\ldots,\varphi(y_n,z_n)]$, and $\y:=[y_1,\ldots,y_n]^\top$. 
    Then, we can estimate $\mathbf{b}$, $\covw$, and $\covxw$ with their empirical counterparts $\hat{\mathbf{b}} := n^{-1}\sum_{i=1}^n y_i\varphi(y_i,z_i) = n^{-1}\Upsilon\mathbf{y}$, $\ecovxw := n^{-1}\sum_{i=1}^n\phi(x_i)\otimes\varphi(y_i,z_i) = n^{-1}\Phi\Upsilon^\top$ and $\ecovw = n^{-1}\sum_{i=1}^n\varphi(y_i,z_i)\otimes\varphi(y_i,z_i) = n^{-1}\Upsilon\Upsilon^\top$.
    We denote the empirical version of \eqref{eq:rkhs-obj} by $\widehat{\Psi}(f,u)$ and the estimate of $f^*$ by $\hat{f}$.

    Next, we show that the representer theorem \citep{Schoelkopf01:Representer} for $\widehat{\Psi}(f,u)$ holds for both $f$ and $u$.
    \begin{restatable}{lemma}{representer}
    \label{lem:representer}
    %Let $\hbspg$ and $\hbspu$ be RKHSes endowed with kernels $k:\inx\times\inx\to\rr$ and $l:\inw\times\inw\to\rr$, respectively. 
    For any $f\in\hbspf$ and $u\in\hbspu$, there exist 
    $f_{\bvec} = \sum_{i=1}^n\beta_i k(x_i,\cdot)$ and $u_{\avec} = \sum_{i=1}^n\alpha_i l(w_i,\cdot)$ for some $\avec,\bvec\in\rr^n$ such that $\widehat{\Psi}(f,u) = \widehat{\Psi}(f_{\bvec},u_{\avec})$.
    \end{restatable}
    
    By virtue of Lemma \ref{lem:representer}, the solution to \eqref{eq:close-form} can be expressed as $f(x) = \sum_{i=1}^n\beta_ik(x_i,x)$ where the coefficients $\bvec$ are given by the following proposition. 
    
    \begin{restatable}{proposition}{empiricalestimate}
    \label{prop:empirical}
    Given an i.i.d. sample $(x_i,y_i,z_i)_{i=1}^n$ from $\mathbb{P}(X,Y,Z)$, let $\kmat := \Phi^\top\Phi$ and $\lmat := \Upsilon^\top\Upsilon$ be the Gram matrices such that $\kmat_{ij} = k(x_i,x_j)$ and $\lmat_{ij} = l(w_i,w_j)$ where $w_i := (y_i,z_i)$.
    Then, $\hat{f} = \Phi\bvec$ where $\bvec = (\M\kmat + n\lambda_2 \kmat)^{-1}\M\y$ and $\M := \kmat(\lmat + n\lambda_1 I)^{-1}\lmat$.
    \end{restatable}

    % Finally, one could argue that the proposed framework still requires the dual function $u$ to be estimated. 
    % However, c
    % Compared to first-stage regression, estimating the dual function is arguably easier as it is a real-valued function. 
    % On the other hand, previous work involved conditional density estimation \citep{Hall05:NIV,Darolles11:NIV,Hartford17:DIV} and vector-valued regression \citep{Singh19:KIV} which become computationally prohibitive when $X$ and $Z$ are high-dimensional.
    Compared to previous work which involved conditional density estimation \citep{Hall05:NIV,Darolles11:NIV,Hartford17:DIV} and vector-valued regression \citep{Singh19:KIV} as first-stage regression, estimating the dual function $u$, a real-valued function, is arguably easier.
    This is especially so when $X$ and $Z$ are high-dimensional.

\begin{algorithm}[t]
    \caption{Kernelized DualIV}
    \label{alg:kernel-dualiv}
    \begin{algorithmic}[1]
        \REQUIRE Data $(x_i,y_i,z_i)_{i=1}^n$, kernel functions $k,l$, and a parameter grid $\Gamma$.
        \STATE Compute kernel matrices $\kmat_{ij}=k(x_i,x_j)$ and $\lmat_{ij} = l((y_i,z_i),(y_j,z_j))$.
        \STATE $(\lambda_1,\lambda_2) \leftarrow \text{\texttt{SelectParams}}(\kmat,\lmat,\Gamma)$.
        \STATE $\M \leftarrow \kmat(\lmat + n\lambda_1 I)^{-1}\lmat$.
        \STATE $\bvec \leftarrow (\M\kmat + n\lambda_2 \kmat)^{-1}\M\y$.
        \ENSURE $f(x) = \sum_{i=1}^n\beta_i k(x_i,x)$.
    \end{algorithmic}
\end{algorithm}

%%%%
\vspace{-5pt}
\paragraph{Hyperparameter selection.}
Our estimator depends on two hyper-parameters, $\lambda_1$ and $\lambda_2$.
Given a dataset $(x_i,y_i,z_i)_{i=1}^{2n}$ of size $2n$, we provide a simple heuristic to determine the values of $(\lambda_1,\lambda_2)$.
Ideally, if we know the optimal dual function $u^*$, we can interpret $u^*(y,z)^2$ as a loss function of $f^*$ at $(y,z)$, as discussed in Section \ref{sec:identifiability}. 
To this end, we first estimate $\hat{f}$ via Proposition \ref{prop:empirical} and $\hat{u}_{\lambda} := (\ecovw + \lambda\mathcal{I})^{-1}(\ecovwx\hat{f} - \hat{\mathbf{b}})$ on the first half of the data $(x_i,y_i,z_i)_{i=1}^n$. 
Next, the out-of-sample loss of $\hat{f}$ is evaluated on the second half $(x_i,y_i,z_i)_{i=n+1}^{2n}$ by
\begin{equation}\label{eq:oos-loss}
\widehat{R}(\hat{f}) = \frac{1}{n}\sum_{i=n+1}^{2n} (\ep_{X|z_i}[\hat{f}(X)] - y_i)^2 \approx \frac{1}{n}\sum_{i=n+1}^{2n} \hat{u}_{\lambda}(y_i,z_i)^2.
\end{equation}
Note that $\hat{u}_{\lambda} = \Upsilon(\lmat + n\lambda I)^{-1}(\kmat\bvec - \y) = \Upsilon\avec$ where $\avec := (\lmat + n\lambda I)^{-1}(\kmat\bvec - \y)$ and $\kmat,\lmat \in\rr^{n\times n}$ are kernel matrices evaluated on $(x_i,y_i,z_i)_{i=1}^n$.
Hence, $\widehat{R}(\hat{f}) \approx \avec^\top\widetilde{\lmat}\mathbf{1}/n$ where $\widetilde{\lmat}_{ij}=l((y_i,z_i),(y_j,z_j))$ for $i=1,\ldots,n$ and $j=n+1,\ldots,2n$ and $\mathbf{1}$ is the all-ones column vector.
In practice, we fix $\lambda$ in $\hat{u}_{\lambda}$ to a small constant to stabilize the loss \eqref{eq:oos-loss} and only optimize $(\lambda_1,\lambda_2)$ that appear in $\bvec$.
Note that this procedure differs from the two-stage causal validation procedures used in \citep{Hartford17:DIV,Singh19:KIV}.
Alternatively, one may choose the hyperparameters by cross-validation with respect to \eqref{eq:oos-loss}.

Algorithm \ref{alg:kernel-dualiv} outlines the kernelized DualIV method whose consistency is studied in Appendix \ref{sec:consistency}.
We note that above algorithm involves matrix inversions ($\mathcal{O}(n^3)$) which become the primary computational bottlenecks when scaling to large datasets.
% as routines have runtime complexities of $\mathcal{O}(n^3)$.
To improve the scalability of our algorithm, we can leverage the rich literature on large-scale kernel machines such as random Fourier features and Nystr\"om method; see, e.g., \citet{NIPS2012_4588} and references therein.
Alternatively, we can employ stochastic gradient descent-based (SGD) algorithms similar to those proposed in \citet[Algorithm 1]{Dai17:Dual} to solve the dual formulation \eqref{eq:1sls-obj} directly. 
This would also allow us to employ flexible models such as neural networks to parameterize the function classes $\mathcal{F}$ and $\mathcal{U}$. 
Recently, \citet{Liao20:NeuralSEM} has taken an important step in this direction and provided convergence analysis for neural networks under a similar formulation.

%%%%
\section{Experiments}
\label{sec:experiments}

In this section, we compare kernelized DualIV\footnote{Our implementation is available at \url{https://github.com/krikamol/DualIV-NeurIPS2020}.} with:
\begin{enumerate*}[label=(\roman*)]
\item vanilla two-stage least squares (2SLS) \citep{Angrist96:IV},
\item DeepIV \citep{Hartford17:DIV},
\item KernelIV \citep{Singh19:KIV} and
\item DeepGMM \citep{Bennett19:DeepGMM}. 
\end{enumerate*}
To provide a fair comparison, we adhered to the provided hyperparameter settings. Given the low-dimensional nature of our experiments, we used DeepGMM's settings for low-dimensional scenarios in \citep[Appendix B.2.1]{Bennett19:DeepGMM}. We ran 20 simulations of each algorithm for sample sizes of $50$ and $1000$ and calculated the $\log_{10}$ mean squared error and its standard deviations w.r.t. the true function $f$ for 2800 out-of-sample test points.

%In addition, since we can also interpret our algorithm as GMM using augmented instrumental variable, we compare it to the recent extensions of GMM:
%\begin{enumerate*}[label=(\roman*)]
%\item \textbf{GMM}: We use the standard implementation of the vanilla GMM as explained in Section \ref{sec:gmm}.
%\item \textbf{DeepGMM} \citep{Bennett19:DeepGMM}: We use the implementation given in \citet{Bennett19:DeepGMM}. \footnote{\url{https://github.com/CausalML/DeepGMM}} Note that we only compare our algorithm to DeepGMM as more comprehensive comparisons between DeepGMM and adversarial GMM \citep{Lewis18:AGMM}  were already available in \citet{Bennett19:DeepGMM}.
%\end{enumerate*}

%%%%
\vspace{-5pt}
\paragraph{Demand design.}
We consider the same simulation as in \citep{Hartford17:DIV,Singh19:KIV}: $Y = f(X) + \varepsilon$ where $Y$ is outcome, $X=(P,T,S)$ are inputs, and $Z=(C,T,S)$ are instruments.
Specifically, $Y$ is sales, $P$ is price, which is endogeneous, $C$ is a supply cost shifter (instrument), and $(T,S)$ are time of year and customer sentiment acting as exogeneous variables.
The aim is to estimate the demand function $f(p,t,s) = 100 + (10+p)s\psi(t)-2p$ where $\psi(t) = 2[(t-5)^4/600 + \exp(-4(t-5)^2)+t/10 -2]$. 
Training data is sampled according to $S\sim \text{Unif}\{1,\ldots,7\}$, $T\sim\text{Unif}[0,10]$, $(C,V)\sim\mathcal{N}(\mathbf{0},I_2)$, $\varepsilon \sim\mathcal{N}(\rho V, 1-\rho^2)$, and $P = 25 + (C+3)\psi(T)+V$.
The parameter $\rho\in\{0.1,0.25,0.5,0.75,0.9\}$ controls the extent to which price $P$ is confounded with outcome $Y$ by supply-side market forces. 
In our notation, $X=(P,T,S)$, $Z=(C,T,S)$ and $W=(Y,Z)=(Y,C,T,S)$. 
%Figure \ref{fig:demand-function} illustrates the demand function.

For DualIV, we used the Gaussian RBF kernel for both $k$ and $l$.
% , \ie, $k(x,x') = \exp(-\|x-x'\|_2^2/2\sigma_x^2)$ and $l(w,w') = \exp(-\|w-w'\|_2^2/2\sigma_w^2)$ where $\sigma_x$ and $\sigma_w$ are bandwidth parameters. 
In the experiments, the kernels on $X$ and $Z$ are product kernels, \ie, $k(x_i,x_j) = k_p(p_i,p_j)k_t(t_i,t_j)k_s(s_i,s_j)$ and $k(z_i,z_j) = k_c(c_i,c_j)k_t(t_i,t_j)k_s(s_i,s_j)$, and $l([y_i, z_i],[y_j, z_j])=\exp([y_i-y_j, z_i - z_j]^\top V_{yz}^{-1}[y_i-y_j, z_i - z_j])$ where $V_{yz}$ is a symmetric bandwidth matrix. % In this experiment, we use the product of kernels $k(x_i,x_j) = k_p(p_i,p_j)k_t(t_i,t_j)k_s(s_i,s_j)$ and $l([y_i, c_i],[y_j, c_j])=\exp([y_i-y_j, c_i - c_j]^\top V_{yc}^{-1}[y_i-y_j, c_i - c_j])$ where $V_{yc}$ is a symmetric bandwidth matrix and $k_p,k_t,k_s$ are all Gaussian kernels.
The values of all bandwidth parameters are determined via the median heuristic.
% $(\lambda_1^*, \lambda_2^*)$ are obtained using the hyperparameter selection procedure described in Section \ref{sec:estimation}. 
We choose $(\lambda_1, \lambda_2)$ from $\{10^{-10},10^{-9},\ldots,10^{-1}\}$ via cross-validation.
% Our training procedure is as follows. We split the training dataset in two. One half is used to fit the model $\hat{f}$ and the other is used to evaluate the out-of-sample risk $\widehat{R}(\hat{f})$.
Once $(\lambda_1, \lambda_2)$ is chosen, we refit $\hat{f}$ on the entire training set.

\begin{table}[t!]
    \centering
    \caption{Comparisons of IV regression methods in small (top) and medium (bottom) sample size regimes. We report the $\log_{10}$ mean squared error (MSE) and its standard deviations over 20 trials.}
    \label{tab:results}
    \resizebox{\textwidth}{!}{
    \begin{tabular}{lccccc}
        \toprule
         \multirow{2}{*}{$n=50$} & \multicolumn{5}{c}{\textbf{$\mathrm{Log}_{10}$ Mean Squared Error (MSE)}} \\
         & $\rho=0.1$ & $\rho=0.25$ & $\rho=0.5$ & $\rho=0.75$ & $\rho=0.9$ \\
         \midrule
         \;\;\texttt{2SLS} &$5.814\pm1.214$&$6.013\pm0.827$&$5.895\pm 0.718$&$5.625\pm1.182$&$5.308\pm1.031$ \\
         \;\;\texttt{DeepIV} &$5.127\pm0.043$&$5.131\pm0.031$&$5.133\pm0.072$&$5.130\pm0.124$&$5.127\pm 0.061$ \\
         \;\;\texttt{KernelIV} &$4.481\pm 0.134$ & $4.460\pm 0.095$ & $4.438\pm 0.132$ & $4.433\pm 0.100$ & $4.462\pm 0.114$ \\
         \;\;\texttt{DeepGMM} &$3.848\pm 1.096$ & $2.899\pm 1.638$ & $3.952\pm 0.900$ & $4.148\pm 0.556$ & $3.738\pm 0.587$ \\
         \;\;\texttt{DualIV} &$4.257\pm 0.108$ & $4.210\pm 0.126$ & $4.285\pm 0.170$ & $4.286\pm 0.126$ & $4.232\pm 0.152$ \\
         \midrule
        % \midrule
         %\multirow{2}{*}{$n=1000$} & \multicolumn{5}{c}{\textbf{$\mathrm{Log}_{10}$ Mean Squared Error (MSE)}} \\
         %& $\rho=0.1$ & $\rho=0.25$ & $\rho=0.5$ & $\rho=0.75$ & $\rho=0.9$ \\
         $n=1000$ & & & & & \\
         %\midrule
         \;\;\texttt{2SLS} & $8.236\pm 0.117$ & $7.242\pm 1.232$ & $8.290\pm 1.132$ & $8.371\pm 0.865$&$8.544\pm 1.109$ \\
         \;\;\texttt{DeepIV} &$4.613\pm0.052$&$4.618\pm0.048$&$4.614\pm0.068$&$4.701\pm0.040$&$4.731\pm0.032$ \\
         \;\;\texttt{KernelIV} &$4.189\pm 0.046$ & $4.209\pm 0.040$ & $4.199\pm 0.043$ & $4.195\pm 0.045$ & $4.194\pm 0.055$ \\
         \;\;\texttt{DeepGMM} &$4.090\pm 0.691$ & $3.953\pm 1.076$ & $4.392\pm 0.561$ & $4.272\pm 0.595$ & $4.415\pm 0.522$ \\
         \;\;\texttt{DualIV} &$4.143\pm 0.117$& $4.221\pm 0.185$ & $4.104\pm 0.102$ & $4.142\pm 0.105$ & $4.127\pm 0.106$ \\
         \bottomrule
    \end{tabular}}
\end{table}

Table \ref{tab:results} reports the results of different methods evaluated on the test data.
First, we observe that 2SLS achieves the largest MSE in both regimes as expected because the linearity assumption is violated here. 
Second, in the small sample size regime, DeepIV achieves relatively larger MSE than the other non-linear methods. 
KernelIV, DeepGMM, and DualIV, on the other hand, have comparable performance, with DeepGMM having the lowest MSE. 
However, we note that the results attained by DeepGMM were unstable out of the box and we had to reduce the variance of the initialization of the neural networks to $0.1$ to obtain some degree of stability which is reflected in the standard deviations. 
We can fully attribute this variability to initialization as DeepGMM's default batch size of 1024 is larger than that of both training datasets so there is no sampling variability in the optimization process. 
This suggests that DeepGMM, like DeepIV, is relatively brittle compared to kernel-based methods in the small sample size regime. 
Furthermore, DeepGMM comes with an extensive hyperparameter selection process, which highlights its need for fine-tuning. 
Last but not least, DualIV is competitive to KernelIV across $\rho$ with slightly smaller MSE, which lends weight to our hypothesis that estimating the real-valued dual function is easier than vector-valued regression.

In the medium sample size regime, we observe that performance of DeepIV is in the same ballpark as the rest of the non-linear IV regression methods and the variance of DeepGMM is reduced, albeit still highest among the non-linear methods. 
The results of DualIV, KernelIV and DeepGMM are almost indistinguishable with DualIV having an edge as $\rho$ increases. 
This could mean accounting for both $Y$ and $Z$ is perhaps slightly more effective than $Z$ alone in the presence of greater confounding.

\section{Conclusion}
\label{sec:discussion}

This paper proposes a general framework for non-linear IV regression called DualIV.
Unlike previous work, DualIV does not require the first-stage regression which is the critical bottleneck of modern two-stage procedures.
By exploiting tools in stochastic programming, we were able to reformulate the two-stage problem as the convex-concave saddle-point problem which is relatively simpler to solve.
Instead of first-stage regression, DualIV requires the dual function $u(y,z)$ to be estimated, which is arguably easier than first-stage regression, especially when the instruments and treatments are high-dimensional.
We demonstrate the validity of our framework with a kernel-based algorithm.
Results show the competitiveness of our algorithm with respect to existing ones.
Finally, potential directions for future work include
\begin{enumerate*}[label=(\roman*)]
\item a minimax convergence analysis which could provide additional insight into the benefits of our framework,
\item more flexible and scalable models such as deep neural networks as dual functions with stochastic gradient descent (SGD) \citep{Dai17:Dual}, and
\item applications to other two-stage problems in causal inference such as double ML \citep{Chernozhukov18:DML}.
\end{enumerate*}

%%%%
\section*{Broader impact}

This work provides a new framework for non-linear instrumental variable regression which allows one to perform causal analysis under the presence of unobserved confounders.
This could have a profound impact in other fields such as economics, social science, and epidemiology, among others.
Understanding the role of instruments in the context of learning theory may also pave the way towards creating more robust and trustworthy machine learning algorithms that are capable of surviving in the world full of hidden biases.

%%%
\begin{ack}
We are indebted to Rahul Singh and Arthur Gretton for their help with the KernelIV code used in our experiments. 
We thank Victor Chernozhukov, Elias Bareinboim, Sorawit Saengkyongam, Uri Shalit, Konrad Kording, Rahul Singh, Arthur Gretton, and You-Lin Chen for fruitful discussions as well as anonymous reviewers for the helpful feedback on our initial submission.

This work was funded by the federal and state governments of Germany through the Max Planck Society (MPG).
\end{ack}

%%% 
\bibliography{refs}
\bibliographystyle{unsrtnat}

%%%%
\clearpage
\appendix

\section{Conjugate loss function}
\label{sec:conjugate-loss}

Let $\ell_y(v) := \frac{1}{2}(y-v)^2$ be a proper, convex, and lower semi-continuous function for all $y\in\mathbb{R}$. It follows from the definition of Fenchel conjugate (see, \eg, Definition \ref{def:fenchel} or \citep[Ch. 14]{Rock98:VA} and \citep[Ch. 7]{Shapiro14:LSP}) that for any $y\in\mathbb{R}$,
\begin{equation}\label{eq:apx:dual-loss}
    \ell^{\star}_y(u) := \sup\left\{ uv - \ell_y(v) \,:\, v\in\mathbb{R} \right\}
    = \sup\left\{ uv - \frac{1}{2}(y-v)^2 \,:\, v\in\mathbb{R} \right\}.
\end{equation}
Hence, $\ell^{\star}_y(u)$ is also a proper, concave, and upper semi-continuous function. Taking a derivative of $uv - \frac{1}{2}(y-v)^2$ w.r.t. $v$ and setting it to zero yield a critical point $v^* = u+y$ for any $u,y\in\mathbb{R}$. Since $uv - \frac{1}{2}(y-v)^2$ is a concave function in $v$, we can substituting $v^*$ back into \eqref{eq:apx:dual-loss} to obtain
\begin{equation*}
    \ell^{\star}_y(u)  = u(u+y) - \frac{1}{2}(y-(u+y))^2 
    = u^2 + uy - \frac{1}{2}u^2 = uy + \frac{1}{2}u^2,
\end{equation*}
as required.

%%%%
\section{Connection to generalized method of moments (GMM)}
\label{sec:gmm-is-dualiv}

To understand the connection between DualIV and GMM,
let us consider $\mathcal{U} := \text{span}(g_1,\ldots,g_m)$ where $g_1,\ldots,g_m$ are arbitrary real-valued functions on $\iny\times\inz$. 
That is, for any $u\in\mathcal{U}$, we have $u = \sum_{j=1}^m\alpha_jg_j$ for some $(\alpha_1,\ldots,\alpha_m)^\top\in\mathbb{R}^m$. Then, we have 
\begin{eqnarray*}
    J(f) &:=& \max_{u\in\mathcal{U}} \Psi(f,u) \\
    &=& \max_{\bm{\alpha}\in\mathbb{R}^m}\, \mathbb{E}_{\mathit{XYZ}}\left[(f(X)-Y)\left( \sum_{j=1}^m\alpha_jg_j(Y,Z)\right)\right] 
    - \frac{1}{2}\mathbb{E}_{\mathit{YZ}}\left[\left(\sum_{j=1}^m\alpha_jg_j(Y,Z)\right)^2\right] \\
        &=& \max_{\bm{\alpha}\in\mathbb{R}^m}\, \sum_{j=1}^m\alpha_j\mathbb{E}_{\mathit{XYZ}}\left[(f(X)-Y)g_j(Y,Z)\right] - \frac{1}{2}\mathbb{E}_{\mathit{YZ}}\left[\left(\sum_{j=1}^m\alpha_jg_j(Y,Z)\right)^2\right] \\
        &=& \max_{\bm{\alpha}\in\mathbb{R}^m}\, \bm{\alpha}^\top\bm{\psi} - \frac{1}{2}\bm{\alpha}^\top\Lambda\bm{\alpha},
\end{eqnarray*}
\noindent where we define $\bm{\alpha}:=(\alpha_1,\ldots,\alpha_m)^\top$, $\bm{\psi}:=(\psi(f,g_1),\ldots,\psi(f,g_m))^\top$, and $\Lambda := \mathbb{E}_{\mathit{YZ}}[\mathbf{g}(Y,Z)\otimes\mathbf{g}(Y,Z)]$ with $\mathbf{g}(Y,Z) := (g_1(Y,Z),\ldots,g_m(Y,Z))^\top$. Taking the derivative w.r.t. $\bm{\alpha}$ and setting it to zero yield
\begin{equation}\label{eq:dual-gmm-obj}
    J(f) = \frac{1}{2}\bm{\psi}^\top\Lambda^{-1}\bm{\psi}.
\end{equation}
In this case, the DualIV objective can be expressed in a quadratic form. 
In the language of GMM, $\bm{\psi}$ acts as a vector of moment conditions and $\Lambda$ acts as a weighting matrix \citep{Hansen82:GMM,Hall05:GMM}.
However, we reiterate that there is a fundamental difference here: the dual function $u\in\mathcal{U}$ cannot act as an instrument since it depends on $Y$ and thereby violates the exclusion restriction assumption.
Recently, \citet[Appendix F]{Liao20:NeuralSEM} provided a clarification on this connection by employing an alternative formulation.

%%%
\section{Dual formulation in RKHS}
\label{sec:apx-kernelized-iv}

In this section, we provide a detailed derivation of the dual formulation \eqref{eq:rkhs-obj} when $\hbspf$ and $\hbspu$ are both reproducing kernel Hilbert spaces (RKHSs) associated with positive definite kernels $k:\inx\times\inx\to\rr$ and $l:\inw\times\inw\to\rr$.
Let $\phi: x\mapsto k(x,\cdot)$ and $\varphi: w \mapsto l(w,\cdot)$ be the canonical feature maps of $k$ and $l$, respectively \citep{Scholkopf01:LKS}.
We assume throughout that both $\hbspf$ and $\hbspu$ are universal such that they are both dense in the space of bounded continuous functions (see, \eg, \citep[Ch. 4]{SteChr2008}).
Furthermore, let $\text{HS}(\hbspf,\hbspu)$ be a Hilbert space of Hilbert-Schmidt operators mapping from $\hbspf$ to $\hbspu$ with an inner product $\langle \cdot,\cdot\rangle_{\mathrm{HS}}$ (see, \eg, \citep[Sec. 2.3]{Muandet17:KME}).
    
Then, for any $f\in\hbspf$ and $u\in\hbspu$, we can rewrite the objective in \eqref{eq:1sls-obj} as a functional
\begin{eqnarray*}
        \Psi(f,u) &=& \ep_{\mathit{XW}}[f(X)u(W)] - \ep_{\mathit{YZ}}[Yu(Y,Z)]  - \frac{1}{2}\ep_{W}[u(W)^2] \\
        &=& \ep_{\mathit{XW}}[\langle f,\phi(X)\rangle_{\hbspf}\langle u,\varphi(W)\rangle_{\hbspu}] - \ep_{\mathit{YZ}}[Y\langle u,\varphi(Y,Z) \rangle_{\hbspu}]
        - \frac{1}{2}\ep_{W}[\langle u,\varphi(W) \rangle_{\hbspu}^2] \\
        &=& \ep_{\mathit{XW}}[\langle f\otimes u,\phi(X)\otimes\varphi(W)\rangle_{\text{HS}}] - \langle u,\ep_{\mathit{YZ}}[Y\varphi(Y,Z)] \rangle_{\hbspu}  \\
        && - \frac{1}{2}\ep_{W}[\langle u\otimes u,\varphi(W)\otimes \varphi(W) \rangle_{\text{HS}}] \\
        &=& \langle \covxw,f\otimes u\rangle_{\text{HS}} - 
        \langle u,\mathbf{b}\rangle_{\hbspu} 
        - \frac{1}{2}\langle \covw,u\otimes u\rangle_{\text{HS}} \\
        &=& \langle f, \covxw u\rangle_{\hbspf} - \langle u,\mathbf{b}\rangle_{\hbspu} - \frac{1}{2}\langle u,\covw u \rangle_{\hbspu},
\end{eqnarray*}
\noindent where $\mathbf{b} := \ep_{\mathit{YZ}}[Y\varphi(Y,Z)]\in\hbspu$, $\covw := \ep_{\mathit{W}}[\varphi(W)\otimes\varphi(W)] \in\hbspu\otimes\hbspu$ is a covariance operator, and $\covxw := \ep_{\mathit{XW}}[\phi(X)\otimes\varphi(W)] \in \hbspf\otimes\hbspu$ is a cross-covariance operator \citep{Baker1973,Fukumizu04:DRS}.
We used the reproducing property of $\hbspf$ and $\hbspu$ in the second equality.
The third equality follows from the property of the rank-one operator, \ie, 
$$\langle f,g\rangle_{\hbspf}\langle v,u \rangle_{\hbspu} = \langle f\otimes v, u \otimes g \rangle_{\text{HS}(\hbspf,\hbspu)}.$$
We then used the definition of $\covw$, $\covxw$, and $\mathbf{b}$ to get the fourth equation.
The last equation follows from the fact that $\langle L, f\otimes u\rangle_{\text{HS}(\hbspu,\hbspf)} = \langle f, Lu\rangle_{\hbspf}$ for any Hilbert-Schmidt operator $L \in \text{HS}(\hbspu,\hbspf)$, $f\in\hbspf$, and $u\in\hbspu$.

\section{Consistency} 
\label{sec:consistency}

In this section, we show that the kernelized DualIV estimator is asymptotically consistent under the assumption that $\covw^{-1}$ and $(\covxw \covw^{-1}\covwx)^{-1}$ exist. 
Under these assumptions, we show in \eqref{eq:close-form} that the solution $f^*$ of the saddle-point problem \eqref{eq:rkhs-obj} can be expressed as 
$f^* = (\covxw \covw^{-1}\covwx)^{-1}\covxw\covw^{-1}\mathbf{b}$. 
For simplicity, we assume further that the operator norm of the inverse covariance functions are bounded from below. 
The following theorem shows the consistency for the estimator $\hat{f}$ obtained via Proposition \ref{prop:empirical}.

%%%
\begin{restatable}{theorem}{consistency}
\label{thm:consistency}
    Let $\hat{f}_{\lambda}$ be an empirical estimator of $f^*$ obtained from Proposition \ref{prop:empirical} with the regularization parameters $\lambda := (\lambda_1,\lambda_2)$.
    Assume that $\covw^{-1}$ and $(\covxw \covw ^{-1}\covwx)^{-1}$ exist and the operator norm of the inverse are bounded. 
    Then, for sufficiently slow decay of regularization parameters $\lambda_1$ and $\lambda_{2}$, $\hat{f}_{\lambda}$ is a consistent estimator of $f^*$ in RKHS norm, \ie, $\|\hat{f}_{\lambda} - f^* \|_{\hbspf} \rightarrow 0$ as $n\rightarrow \infty$.
\end{restatable}

The proof of this theorem can be found in Appendix \ref{sec:consistency-proof}.

The critical drawback of Theorem \ref{thm:consistency} is that it assumes the existence of $\covw^{-1}$ and $(\covxw \covw^{-1}\covwx)^{-1}$ which may not hold in general.
Similar assumption was also made in \citet{Fukumizu13:KBR} who provided counterexamples of cases in which such an assumption does not hold; see, also,  \citep{Song10:KCOND,Fukumizu13:KBR,Muandet17:KME} for the detailed discussion.
One potential direction for future work is thus to provide the consistency result of DualIV under weaker assumptions.

%%%
\section{Proofs}
\label{sec:proofs}

This section contains the detailed proofs.

%%%%
\subsection{Proof of Proposition \ref{prop:saddle-solution}}
\label{sec:proof-saddle-sol}

\saddlepoint*

%%%
\begin{proof}
    Taking $\ell$ in \eqref{eq:erm} to be $\frac{1}{2}(y - y')^2$, plugging $u^*(y,z) = \ep_{X|z}[f(X)] - y$ into \eqref{eq:1sls-obj} yields
    \begin{align*}
    \Psi(f,u^*)
        &= \ep_{\mathit{XYZ}}[(f(X) - Y)u^*(Y,Z)] - \frac{1}{2}\ep_{\mathit{YZ}}[u^*(Y,Z)^2] \\
        &= \ep_{\mathit{XYZ}}[(f(X) - Y)(\ep_{X|Z}[f(X)] - Y)]
        - \frac{1}{2}\ep_{\mathit{YZ}}[(\ep_{X|Z}[f(X)] - Y)^2] \\
        &= \ep_{\mathit{YZ}}[(\ep_{X|Z}[f(X)] - Y)(\ep_{X|Z}[f(X)] - Y)] 
        - \frac{1}{2}\ep_{\mathit{YZ}}[(\ep_{X|Z}[f(X)] - Y)^2] \\
        &= \ep_{\mathit{YZ}}[(\ep_{X|Z}[f(X)] - Y)^2] 
        - \frac{1}{2}\ep_{\mathit{YZ}}[(\ep_{X|Z}[f(X)] - Y)^2] \\
        &= \frac{1}{2}\ep_{\mathit{YZ}}[(\ep_{X|Z}[f(X)] - Y)^2] = R(f),
    \end{align*}
    as required.
\end{proof}

%%%
\subsection{Proof of Lemma \ref{lem:representer}}
\label{sec:representer}

\representer*

\begin{proof}
Given a fixed sample $(x_i,y_i,z_i)_{i=1}^n$ of size $n$, any RKHSes $\hbspf$ and $\hbspu$ can be decomposed as $\hbspf = \hbspf_{\bvec}\oplus\hbspf_{\perp}$ and $\hbspu = \hbspu_{\avec}\oplus\hbspu_{\perp}$ where $\hbspf_{\bvec}$ and $\hbspu_{\avec}$ are respectively subspaces consisting of functions of the following forms:
    \begin{equation*}
        f_{\bvec} = \sum_{i=1}^n\beta_i k(x_i,\cdot), \quad u_{\avec} = \sum_{i=1}^n\alpha_i l(w_i,\cdot),
    \end{equation*}
    \noindent for some $\bvec\in\rr^n$ and $\avec\in\rr^n$. 
    The orthogonal subspaces $\hbspf_{\perp}$ and $\hbspu_{\perp}$ consist of elements which are orthogonal to $\hbspf_{\bvec}$ and $\hbspu_{\avec}$, respectively, \ie, for any $f_{\bvec}\in\hbspf_{\bvec}$, $f_{\perp}\in\hbspf_{\perp}$, $u_{\avec}\in\hbspu_{\avec}$, $u_{\perp}\in\hbspu_{\perp}$, we have
    $\langle f_{\bvec},f_{\perp}\rangle_{\hbspf} = 0$ and $\langle u_{\avec},u_{\perp}\rangle_{\hbspu} = 0$.
    Any elements $f\in\hbspf$ and $u\in\hbspu$ can thus be expressed as $f = f_{\bvec}+f_{\perp}$ and $u = u_{\avec} + u_{\perp}$ where $f_{\bvec}\in\hbspf_{\bvec}$, $f_{\perp}\in\hbspf_{\perp}$, $u_{\avec}\in\hbspu_{\avec}$, and $u_{\perp}\in\hbspu_{\perp}$.
    
    Next, recall that 
    $$\widehat{\Psi}(f,u) = \langle f, \ecovxw u\rangle_{\hbspf} - \langle u,\hat{\mathbf{b}}\rangle_{\hbspu} - \frac{1}{2}\langle u,\ecovw u \rangle_{\hbspu}$$ 
    where $\ecovxw = n^{-1}\Phi\Upsilon^\top$, $\ecovw = n^{-1}\Upsilon\Upsilon^{\top}$, $\hat{\mathbf{b}} = n^{-1}\Upsilon\mathbf{y}$, $\Phi = [k(x_1,\cdot),\ldots,k(x_n,\cdot)]$, $\Upsilon = [l(w_1,\cdot),\ldots,l(w_n,\cdot)]$, and $\mathbf{y} = [y_1,\ldots,y_n]^\top$. 
    Using the above decomposition, we have
    \begin{eqnarray*}
        \widehat{\Psi}(f,u) &=& \big\langle f, \ecovxw u \big\rangle_{\hbspf} - \big\langle u,\hat{\mathbf{b}}\big\rangle_{\hbspu} - \frac{1}{2}\big\langle u,\ecovw u \big\rangle_{\hbspu} \\
        &=& \big\langle f_{\bvec}+f_{\perp}, \sum_{i=1}^n\beta'_i k(x_i,\cdot) \big\rangle_{\hbspf} - \big\langle u,\hat{\mathbf{b}}\big\rangle_{\hbspu} - \frac{1}{2}\big\langle u,\ecovw u \big\rangle_{\hbspu} \\
        &=& \big\langle f_{\bvec}, \sum_{i=1}^n\beta'_i k(x_i,\cdot)\big\rangle_{\hbspf} - \big\langle u,\hat{\mathbf{b}}\big\rangle_{\hbspu} - \frac{1}{2}\big\langle u,\ecovw u \big\rangle_{\hbspu},
    \end{eqnarray*}
    \noindent where $\beta'_i := n^{-1}\langle l(w_i,\cdot),u\rangle_{\hbspu}$. Since the choice of $u$ is arbitrary, the minimizer of $\widehat{\Psi}(f,u)$ with respect to $f$ lives in the subspace $\hbspf_{\bvec}$.
    
    Similarly, we can write $\widehat{\Psi}(f,u)$ for any $f\in\hbspf$ as a function of $u\in\hbspu$ as 
    \begin{eqnarray*}
        \widehat{\Psi}(f,u) &=& \big\langle \ecovwx f, u \big\rangle_{\hbspu} - \big\langle u,\hat{\mathbf{b}}\big\rangle_{\hbspu} - \frac{1}{2}\big\langle u,\ecovw u \big\rangle_{\hbspu} \\
        &=& \big\langle \sum_{i=1}^n\alpha'_il(w_i,\cdot), u_{\avec}+u_{\perp}\big\rangle_{\hbspu} 
        - \big\langle u_{\avec}+u_{\perp},\sum_{i=1}^n\alpha_i''l(w_i,\cdot)\big\rangle_{\hbspu} \\
        && - \frac{1}{2}\big\langle u_{\avec}+ u_{\perp},\ecovw (u_{\avec}+u_{\perp}) \big\rangle_{\hbspu} \\
        &=& \big\langle \sum_{i=1}^n\alpha'_il(w_i,\cdot), u_{\avec}\big\rangle_{\hbspu} 
        - \big\langle u_{\avec},\sum_{i=1}^n\alpha_i''l(w_i,\cdot)\big\rangle_{\hbspu} 
        - \frac{1}{2}\big\langle u_{\avec},\ecovw u_{\avec}\big\rangle_{\hbspu}.
    \end{eqnarray*}
    The first equality follows from $\langle f, \ecovxw u\rangle_{\hbspf} = \langle \ecovwx f, u\rangle_{\hbspu}$ as $\ecovxw$ is an adjoint operator of $\ecovwx$. 
    Since the choice of $f$ is arbitrary, the maximizer of $\widehat{\Psi}(f,u)$ with respect to $u$ also lives in the subspace $\hbspu_{\avec}$.
    
    Consequently, $\widehat{\Psi}(f,u) = \widehat{\Psi}(f_{\bvec},u_{\avec})$ for some $\bvec\in\rr^n$ and $\avec\in\rr^n$. This completes the proof.
\end{proof}

%%%
\subsection{Proof of Proposition \ref{prop:empirical}}
\label{sec:empirical-estimate}

\empiricalestimate*

\begin{proof}
    It follows from \eqref{eq:close-form} that the structural function $f^*\in\hbspf$ satisfies
    $$(\covxw(\covw + \lambda_1\mathcal{I})^{-1}\covwx + \lambda_2\mathcal{I})f^* = \covxw(\covw + \lambda_1\mathcal{I})^{-1}\mathbf{b}.$$
    Replacing the population quantities with the empirical counterparts $\ecovxw = \frac{1}{n}\Phi\Upsilon^\top$, $\ecovw = \frac{1}{n}\Upsilon\Upsilon^\top$, and $\hat{\mathbf{b}} = \frac{1}{n}\Upsilon\y$ yields
    \begin{equation*}
     (\Phi\Upsilon^\top(\Upsilon\Upsilon^\top + n\lambda_1\mathcal{I})^{-1}\Upsilon\Phi^\top + n\lambda_2\mathcal{I})f^*
        = \Phi\Upsilon^\top(\Upsilon\Upsilon^\top + n\lambda_1\mathcal{I})^{-1}\Upsilon\y.
    \end{equation*}
    Using the identity $\Upsilon^\top(\Upsilon\Upsilon^\top + n\lambda_1\mathcal{I})^{-1} = (\Upsilon^\top\Upsilon + n\lambda_1 I)^{-1}\Upsilon^\top$, the above equation can be rewritten as
    \begin{eqnarray*}
         (\Phi(\Upsilon^\top\Upsilon + n\lambda_1I)^{-1}\Upsilon^\top\Upsilon\Phi^\top + n\lambda_2\mathcal{I})f^*
        &=& \Phi(\Upsilon^\top\Upsilon + n\lambda_1 I)^{-1}\Upsilon^\top\Upsilon\y \\
        (\Phi(\lmat + n\lambda_1I)^{-1}\lmat\Phi^\top + n\lambda_2\mathcal{I})f^*
        &=& \Phi(\lmat + n\lambda_1 I)^{-1}\lmat\y.
    \end{eqnarray*}
    By Lemma \ref{lem:representer}, $f^* = \Phi\bvec$ for some $\bvec\in\rr^n$. Substituting this back into the equation above yields
    \begin{eqnarray*}
        \Phi(\lmat + n\lambda_1I)^{-1}\lmat\Phi^\top\Phi\bvec + n\lambda_2\Phi\bvec
        &=& \Phi(\lmat + n\lambda_1 I)^{-1}\lmat\y \\
        \Phi(\lmat + n\lambda_1I)^{-1}\lmat\kmat\bvec + n\lambda_2\Phi\bvec
        &=& \Phi(\lmat + n\lambda_1 I)^{-1}\lmat\y.
    \end{eqnarray*}
    Multiplying both sides of the equation by $\Phi^\top$ gives
    \begin{eqnarray*}
        \Phi^\top\Phi(\lmat + n\lambda_1I)^{-1}\lmat\kmat\bvec + n\lambda_2\Phi^\top\Phi\bvec
        &=& \Phi^\top\Phi(\lmat + n\lambda_1 I)^{-1}\lmat\y \\
        \kmat(\lmat + n\lambda_1I)^{-1}\lmat\kmat\bvec + n\lambda_2\kmat\bvec
        &=& \kmat(\lmat + n\lambda_1 I)^{-1}\lmat\y \\
        (\kmat(\lmat + n\lambda_1I)^{-1}\lmat\kmat + n\lambda_2\kmat)\bvec
        &=& \kmat(\lmat + n\lambda_1 I)^{-1}\lmat\y.
    \end{eqnarray*}
    Setting $\M = \kmat(\lmat + n\lambda_1I)^{-1}\lmat$ yields the result.
\end{proof}

%%%
\subsection{Proof of Theorem \ref{thm:consistency}}
\label{sec:consistency-proof}

\consistency*

For ease of understanding, we will use the following notation throughout the proof:
\begin{equation*}
\begin{aligned}
     &\covr &:=& \; \covxw\covw^{-1}\covwx  \\
     &\covrl &:=& \; \covxw(\covw + \lambda_1\mathcal{I})^{-1}\covwx  \\
     &\covl &:=& \; \covw + \lambda_1\mathcal{I} 
\end{aligned}
\hspace{3em}
\begin{aligned}
    &\ecovr &:=& \; \ecovxw\ecovw^{-1}\ecovwx \\ 
    &\ecovrl &:=& \; \ecovxw(\ecovw + \lambda_1\mathcal{I})^{-1}\ecovwx \\
    &\ecovl &:=& \; \ecovw + \lambda_1\mathcal{I} .
\end{aligned}
\end{equation*}

The following identity will be used heavily in our proof:
\begin{equation}\label{eq:inv_diff}    
(B^{-1} - A^{-1}) = B^{-1}(A-B)A^{-1}. 
\end{equation}

\begin{proof}
First, it follows from the assumption that $\|\covw^{-1} \|_{op} \leq \delta_1^{-1}$ and $\|(\covxw \covw^{-1}\covwx^\top)^{-1} \|_{op} \leq \delta_2^{-1}$ for some $\delta_1,\delta_2 > 0$. 
Moreover, we can write the empirical estimate as 
$$\hat{f}_{\lambda} = (\ecovrl +\lambda_2 \Id)^{-1} \ecovxw\ecovl^{-1}\bh .$$
Similarly, under our assumption, the true population function can be expressed as 
$$f^* = \covr^{-1} \covxw \covw^{-1} \bb.$$
The goal is then to bound the difference of $\hat f_{\lambda}$ and $f^*$ in RKHS norm, \ie,
\begin{eqnarray}\label{eq:rkhs-bound}
  \left\| \hat f_{\lambda} - f^* \right\|_{_{\hbspf}} &=&
    \left\| (\ecovrl +\lambda_2 \Id)^{-1} \ecovxw\ecovl^{-1}\bh - \covr^{-1} \covxw \covw^{-1} \bb \right\|_{\hbspf} \nonumber \\
    & \leq & \left\| (\ecovrl +\lambda_2 \Id)^{-1} \ecovxw\ecovl^{-1}\bh - (\covr +\lambda_2 I)^{-1}  \covxw \covl^{-1}\bb \right\|_{\hbspf}  \nonumber \\
    && + \left\| (\covrl +\lambda_2 \Id)^{-1} \covxw \covl^{-1} \bb -  \covr^{-1}\covxw \covw^{-1} \bb \right\|_{\hbspf} \nonumber \\
    &=:& T_1 + T_2.
\end{eqnarray}

\paragraph{Bounding $T_2$:}
Let us first consider the second term $T_2$ in \eqref{eq:rkhs-bound}: 
\begin{eqnarray}\label{eq:T2}
       T_2 & := & \left\| (\covrl +\lambda_2 \Id)^{-1} \covxw \covl^{-1} \bb - 
    \covr^{-1} \covxw \covw^{-1} \bb \right\|_{\hbspf} \nonumber \\
    &\leq&  \left\| (\covrl +\lambda_2 \Id)^{-1} \covxw \covl^{-1} \bb - 
    \covrl^{-1} \covxw \covl^{-1} \bb \right\|_{\hbspf} \nonumber \\
    && + \left\|\covrl^{-1} \covxw \covl^{-1} \bb - \covr^{-1} \covxw \covw^{-1} \bb \right\|_{\hbspf} \nonumber \\
    &=:& T_{21} + T_{22} .
\end{eqnarray}

\paragraph{Bounding $T_{21}$:}
Let us consider $T_{21}$ in \eqref{eq:T2} first. 
\begin{eqnarray}\label{eq:T21}
    T_{21} &=& \left\| (\covrl+\lambda_2 \Id)^{-1} \covxw \covl^{-1} \bb - \covrl^{-1} \covxw \covl^{-1} \bb \right\|_{\hbspf} \nonumber \\
      &=& \lambda_2 \left\| (\covrl +\lambda_2 \Id)^{-1} \covrl^{-1} \covxw \covl^{-1} \bb \right\|_{\hbspf} \nonumber \\
      &\leq& \lambda_2 \left\| (\covrl + \lambda_2 \Id)^{-1} \right\|_{op} \left\|\covrl^{-1} \right\|_{op} \left\|\covxw\covl^{-1} \bb \right\|_{\hbspf} \nonumber \\
      &\leq& \lambda_2 \left\| (\covrl + \lambda_2 \Id)^{-1} (\covxw \covl^{-1} \covwx ) \right\|_{op} \left\|\covrl^{-1}\right\|_{op}^2 \left\|\covxw \covl^{-1} \bb \right\|_{\hbspf} \nonumber \\
      &\leq& \lambda_2 \left\|\covrl^{-1} \right\|_{op}^2 \left\|\covxw \covl^{-1} \bb \right\|_{\hbspf} \nonumber \\
      &\leq& \frac{\lambda_2}{\delta_2^2}\left\|\covxw \covl^{-1} \bb \right\|_{\hbspf} \leq \frac{\lambda_2}{\delta_1 \delta_2^2} \left\|\covxw \|_{op} \right\|\bb \|_{\hbspf}.
\end{eqnarray}
The second equality in \eqref{eq:T21} follows from the identity in \eqref{eq:inv_diff}. 
Hence, we have $T_{21} = \mathcal{O}(\lambda_2)$. 
From the above argument, it is also clear that there exists a positive constant $C$ such that $\max (\|\covxw \covw^{-1} \bb \|_{\hbspf}, \| \covxw\covl^{-1} \bb \|_{\hbspf}) \leq C$. 

\paragraph{Bounding $T_{22}$:}
Let us now consider the term $T_{22}$ in \eqref{eq:T2}.
\begin{equation}
    \begin{aligned}
        T_{22} &= \left\| \covrl^{-1}\covxw \covl^{-1} \bb -  \covr^{-1} \covxw \covw^{-1} \bb \right\|_{\hbspf} \\
        & \leq \left\|  \covrl^{-1} \covxw \covl^{-1} \bb - \covr^{-1}\covxw \covl^{-1} \bb \right\|_{\hbspf} 
        + \left\| \covr^{-1} \covxw \covl^{-1} \bb - \covr^{-1} \covxw \covw^{-1} \bb \right\|_{\hbspf}  \\
        &= \left\| ( \covrl^{-1} - \covr^{-1} ) \covxw \covl^{-1} \bb \right\|_{\hbspf}
         + \left\|\covr^{-1} \big(\covxw \covl^{-1} \bb - \covxw \covw^{-1} \bb \big)  \right\|_{\hbspf} \\
        & \leq C \left\|\covrl^{-1} \right\|_{op} \left\|\covr^{-1} \right\|_{op}  \left\|\covrl - \covr \right\|_{op}
         + \left\|\covr^{-1}\right\|_{op} \left\|\covxw \big(\covl^{-1} - \covw^{-1}\big) \bb \right\|_{\hbspf} \\
        & \leq C \left\|\covrl^{-1} \right\|_{op}\left\|\covr^{-1} \right\|_{op} \left\|\covxw (\covl^{-1} - \covw^{-1} ) \covwx \right\|_{op} \\
        & \quad + \left\|\covr^{-1} \right\|_{op} \left\|\covxw \big(\covl^{-1} - \covw^{-1}\big) \bb  \right \|_{\hbspf} .
    \end{aligned}
\end{equation}

Further, we have
\begin{equation}
       T_{22} \leq  \frac{\lambda C}{\delta_2} \left\|\covrl^{-1} \right\|_{op}   \left\|\covxw \covl^{-1}  \covw^{-1} \covwx \right \|_{op}
         + \left\|\covr^{-1}\right\|_{op} \left\|\covxw \big(\covl^{-1} - \covw^{-1}\big) \bb  \right \|_{\hbspf} .
\end{equation}
Let us consider now the following term in the above inequality:
\begin{eqnarray*}
    \left\|\covrl^{-1} \right\|_{op} &\leq&  \left\|\covr^{-1} \right\|_{op} + \left\|\covrl^{-1}  -  \covr^{-1} \right\|_{op}  \\
    &\leq& \frac{1}{\delta_2} + \left\|\covrl^{-1}  -  \covr^{-1} \right\|_{op}  \\
    &\leq&  \frac{1}{\delta_2} + \left\|\covrl^{-1} \right\|_{op}  \left\|\covr^{-1} \right\|_{op}  \left\|\covrl - \covr \right\|_{op}  \\
    &=&  \frac{1}{\delta_2} + \left\|\covrl^{-1} \right\|_{op}  \left\|\covr^{-1} \right\|_{op}  \left\|\covxw (\covl^{-1} - \covw^{-1} ) \covwx \right\|_{op}  \\
    &=&  \frac{1}{\delta_2} + \lambda_1\left\|\covrl^{-1} \right\|_{op}  \left\|\covr^{-1} \right\|_{op} \left\|\covxw \covl^{-1} \covw^{-1} \covwx \right\|_{op} .
\end{eqnarray*}

Now, $\|\covxw \covl^{-1}  \covw^{-1} \covwx  \|_{op} \leq \tilde{c}/\delta_1^2$ and similarly $\|\covxw \big(\covl^{-1} - \covw^{-1}\big) \bb   \|_{\hbspf} \leq \hat{c}/\delta_1^2$ for some positive real numbers $\hat c$ and $\tilde c$. 
Hence,
\begin{align*}
    \begin{split}
        \left\|\covrl^{-1} \right\|_{op} \leq \frac{1}{\delta_2 }  + \frac{\lambda_1 \tilde c}{\delta_1^2 \delta_2}\left\|\covrl^{-1} \right\|_{op} 
        \quad \Rightarrow \quad 
        \left\|\covrl^{-1} \right\|_{op} \leq \frac{1/\delta_2}{1 - \frac{\lambda_1 \tilde c}{\delta_1^2 \delta_2}}.
    \end{split}
\end{align*}

Hence, if $\lambda_1 \rightarrow 0$ sufficiently fast, then 
\begin{align}
    T_{22} &\leq \frac{\lambda_1 \tilde C}{\delta_1^2\delta_2} 
\end{align}
for a positive real number $\tilde C$. This implies $T_{22}= \mathcal{O}(\lambda_1)$. 

\paragraph{Bounding $T_1$:}
We now consider the first term in \eqref{eq:rkhs-bound}. 
\begin{eqnarray}\label{eq:T1}
       T_1 &=&  \left\| (\ecovrl +\lambda_2 \Id)^{-1} \ecovxw\ecovl^{-1}\bh -  (\covrl +\lambda_2 \Id)^{-1} \covxw \covl^{-1} \bb \right\|_{\hbspf}  \nonumber \\
    &\leq& \left\| (\ecovrl +\lambda_2 \Id)^{-1} \ecovxw\ecovl^{-1}\bh -  (\covrl +\lambda_2 \Id)^{-1}  \ecovxw \ecovl^{-1} \bh \right\|_{\hbspf} \nonumber \\
    && + \left\| (\covrl +\lambda_2 \Id)^{-1} \ecovxw\ecovl^{-1}\bh - (\covrl +\lambda_2 \Id)^{-1} \covxw \covl^{-1} \bb \right\|_{\hbspf} \nonumber \\
    & =:& T_{11} + T_{12}.
\end{eqnarray}

\paragraph{Bounding $T_{11}$:}
Consider the first term in \eqref{eq:T1}:
\begin{align}\label{eq:T11}
    \begin{split}
       T_{11} & = \left\| (\ecovrl +\lambda_2 \Id)^{-1} \ecovxw\ecovl^{-1}\bh -  (\covrl +\lambda_2 \Id)^{-1}  \ecovxw \ecovl^{-1} \bh \right\|_{\hbspf}  \\
    & \leq \left\| \ecovxw \ecovl^{-1} \bh \right\|_{\hbspf} \left\| (\ecovrl +\lambda_2 \Id)^{-1} \right\|_{op}\left\| (\covrl +\lambda_2 \Id)^{-1} \right\|_{op} \left\| \ecovrl - \covrl \right\|_{op} \\
   &\leq \frac{C}{\lambda_1 \lambda_2^2}  \left\| \ecovrl- \covrl \right\|_{op}
    \end{split}
\end{align}
for some positive constant $C$. Next, we have
\begin{align}
    \begin{split}
    \MoveEqLeft \left\| \ecovxw \ecovl^{-1} \ecovwx - \covxw \covl^{-1} \covwx \right\|_{op} \\
      & \leq \left\| \ecovxw \ecovl^{-1} \ecovwx - \ecovxw \covl^{-1} \ecovwx \right\|_{op} + \left\| \ecovxw \covl^{-1} \ecovwx - \covxw \covl^{-1} \covwx \right\|_{op} \\
      & \leq \left\| \ecovxw (\ecovl^{-1} - \covl^{-1}) \ecovwx \right\|_{op}  + \left\| \ecovxw \covl^{-1} \ecovwx - \covxw \covl^{-1} \ecovwx \right\|_{op} \\
      & \qquad + \left\| \covxw \covl^{-1} \ecovwx - \covxw \covl^{-1} \covwx \right\|_{op} \\
      & \leq \left\| \ecovxw \ecovl^{-1}  (\ecovw - \covw)  \covl^{-1} \ecovwx  \right\|_{op} + \left\| (\ecovxw - \covxw) \covl^{-1} \ecovwx \right\|_{op} \\
      & \qquad + \left\| \covxw \covl^{-1} (\ecovwx-\covwx) \right\|_{op}.
    \end{split}
\end{align}
From the $\sqrt{n}$-consistency of covariance and cross-covariance operators \citep{Baker1973,Fukumizu04:DRS}, we have
\begin{equation*}
 \left\| \ecovxw \ecovl^{-1} \ecovwx - \covxw \covl^{-1} \covwx \right\|_{op} = \mathcal{O}\left(\frac{1}{\lambda_1 \sqrt{n}}\right).
\end{equation*}
Hence, if $(\lambda_1\lambda_2)^2$ converges to zero slower than $1/\sqrt{n}$, then $T_{11}$ converges to zero asymptotically. 

\paragraph{Bounding $T_{12}$:}
Let us now consider the second term in \eqref{eq:T1}:
\begin{align}
    \begin{split}
       T_{12}  &=  \left\| (\covrl +\lambda_2 \Id)^{-1} \ecovxw\ecovl^{-1}\bh -  (\covrl +\lambda_2 \Id)^{-1} \covxw \covl^{-1} \bb \right\|_{\hbspf}  \\
    & \leq \left\|(\covrl +\lambda_2 \Id)^{-1} \|_{op} \| \ecovxw\ecovl^{-1}\bh  - \covxw \covl^{-1} \bb \right\|_{\hbspf}  \\
    & \leq \frac{1}{\lambda_2} \left\| \ecovxw\ecovl^{-1}\bh  - \covxw \covl^{-1} \bb \right\|_{\hbspf} \\
    & \leq \frac{1}{\lambda_2} \left[  \left\| \ecovxw\ecovl^{-1}\bh - \covxw\ecovl^{-1}\bh \right\|_{\hbspf}  + \left\| \covxw\ecovl^{-1}\bh - \covxw \covl^{-1}\bh \right\|_{\hbspf} \right. \\
    & \quad\qquad \left. + \left\|\covxw \covl^{-1}\bh - \covxw \covl^{-1} \bb \right\|_{\hbspf} \right] \\
    & \leq \frac{1}{\lambda_2} \left[  \left\| (\ecovxw - \covxw)\ecovl^{-1}\bh  \right\|_{\hbspf}  + \left\| \covxw(\ecovl^{-1} - \covl^{-1})\bh \right\|_{\hbspf} \right. \\
   & \quad\qquad \left.+ \left\|\covxw \covl^{-1}\bh - \covxw \covl^{-1} \bb \right\|_{\hbspg} \right].
    \end{split}
\end{align}
By the $\sqrt{n}$-consistency of mean element and covariance operator in RKHS \citep{Baker1973,Fukumizu04:DRS}, we have that
$T_{12}= \mathcal{O}\left(\frac{1}{\lambda_1\lambda_2 \sqrt{n}}\right)$.
Moreover, it follows from what we have shown so far that
\begin{align}
\begin{split}
     \MoveEqLeft \left\| (\ecovxw\ecovl^{-1}\ecovwx +\lambda_2 \Id)^{-1} \ecovxw\ecovl^{-1}\bh - (\covxw \covw^{-1}\covwx)^{-1} \covxw \covw^{-1} \bb \right\|_{\hbspf} \\ 
    &\leq T_1 + T_2 \leq T_{11} + T_{12} + T_{13} + T_{14}.
\end{split}
\end{align}
Hence, if $\lambda_1$ and $\lambda_2$ converge to zero with the sample size $n$ such that $\frac{1}{\lambda_1^2\lambda_2^2\sqrt{n}}$ also converges to zero, then $\| \hat f_{\lambda} - f^*\|_{_{\hbspf}} \to 0$. 
That is, our estimator is consistent in RKHS norm.
\end{proof}

%%%
\end{document}